\pdfoutput=1

\documentclass[11pt]{article}

\usepackage{acl}

\usepackage{times}
\usepackage{latexsym}

\usepackage[T1]{fontenc}

\usepackage[utf8]{inputenc}

\usepackage{microtype}

\usepackage{inconsolata}

\usepackage[font=small,labelfont=bf]{subcaption}
\usepackage{graphicx}
\usepackage{booktabs}
\usepackage{multirow}
\usepackage{color}
\usepackage{hyperref}
\usepackage{amsmath}
\usepackage{amsfonts}
\usepackage{amsthm}
\usepackage{scalefnt}
\usepackage{pgfplots}
\usepgfplotslibrary{fillbetween}
\usetikzlibrary{intersections, pgfplots.statistics}
\pgfplotsset{compat=1.18}

\newtheorem{hypothesis}{Hypothesis}

\newtheorem{proposition}{Proposition}
\newtheorem{prediction}{Prediction}

\newcommand{\defn}[1]{\textbf{#1}}

\newcommand{\defeq}[0]{\mathrel{\stackrel{\textnormal{\tiny def}}{=}}}

\usepackage{listings}
\usepackage{xifthen}
\usepackage{enumitem}
\definecolor{ETHBlue}{RGB}{33,92,175}	%
\definecolor{ETHGreen}{RGB}{98,115,19}		%
\definecolor{ETHPurpleDark}{RGB}{140,10,89}	%
\definecolor{ETHPurple}{RGB}{163,7,116}	%
\definecolor{ETHGray}{RGB}{111,111,111}	%
\definecolor{ETHRed}{RGB}{183,53,45}	%
\definecolor{ETHPetrol}{RGB}{0,120,148}	%
\definecolor{ETHBronze}{RGB}{142,103,19}	%

\newtoggle{color-macro}
\settoggle{color-macro}{false} %

\iftoggle{color-macro}{
\colorlet{MacroColor}{ETHPetrol}
}{
\colorlet{MacroColor}{black}
}

\newcommand{\mymacro}[1]{{\color{MacroColor} #1}}

\newcommand{\Y}{\mymacro{Y}}
\newcommand{\bfR}{\mymacro{\mathbf{R}}}
\newcommand{\bfP}{\mymacro{\mathbf{P}}}
\newcommand{\bfr}[1][]{\ifthenelse{\isempty{#1}}{\mymacro{\mathbf{r}}}{\mymacro{\mathbf{r}_{#1}}}}
\newcommand{\bfp}[1][]{\ifthenelse{\isempty{#1}}{\mymacro{\mathbf{p}}}{\mymacro{\mathbf{p}_{#1}}}}
\newcommand{\bW}{\mymacro{\mathbf{W}}}
\newcommand{\embedding}{\mymacro{\mathbf{E}}}

\newcommand{\bp}{\mymacro{\boldsymbol{p}}}
\newcommand{\bx}[1][]{\ifthenelse{\isempty{#1}}{\mymacro{\boldsymbol{x}}}{\mymacro{\boldsymbol{x}_{#1}}}}
\newcommand{\bw}[1][]{\ifthenelse{\isempty{#1}}{\mymacro{\boldsymbol{w}}}{\mymacro{\boldsymbol{w}_{#1}}}}
\newcommand{\R}{\mymacro{\mathbb{R}}}

\newcommand{\eos}{\mymacro{\textsc{eos}}}
\newcommand{\alphabet}{\mymacro{\Sigma}}
\newcommand{\alphabeteos}{\mymacro{\overline{\Sigma}}}

\newcommand{\lm}{\mymacro{p}}
\newcommand{\w}[1][]{\ifthenelse{\isempty{#1}}{\mymacro{w}}{\mymacro{w_{#1}}}}
\newcommand{\y}[1][]{\ifthenelse{\isempty{#1}}{\mymacro{y}}{\mymacro{y_{#1}}}}
\newcommand{\delim}{\mymacro{\natural}}
\newcommand{\classtoresponse}{\mymacro{o}}

\newcommand{\template}{\mymacro{t}}

\newcommand{\g}{\mymacro{g}}
\newcommand{\task}{\mymacro{\tau}}
\newcommand{\demonstration}{\mymacro{\boldsymbol{d}}}
\newcommand{\h}{\mymacro{h}}
\newcommand{\sul}{\mymacro{\task_{\text{RA}}}}
\newcommand{\sui}{\mymacro{\task_{\text{PA}}}}
\newcommand{\taskg}{\mymacro{\task_{\g}}}

\newcommand{\pearson}{\mymacro{\gamma_{\text{p}}}}
\newcommand{\spearman}{\mymacro{\gamma_{\text{s}}}}

\newcommand{\tasks}{\mymacro{\mathcal{T}}}
\newcommand{\obtasks}{\mymacro{\overline{\tasks}}}

\newcommand{\syn}[1][]{\ifthenelse{\isempty{#1}}{\mymacro{\g_{\text{syn}}}}{\mymacro{\g_{\text{syn}}^{#1}}}}

\newcommand{\synonym}{\mymacro{\texttt{synonym}}}
\newcommand{\antonym}{\mymacro{\texttt{antonym}}}
\newcommand{\keyword}{\mymacro{\texttt{keyword}}}
\newcommand{\random}{\mymacro{\texttt{random}}}

\newcommand{\promptlm}{\mymacro{\pi}}
\newcommand{\responselm}{\mymacro{\rho}}

\newcommand{\br}{\mymacro{\boldsymbol{r}}}

\newcommand{\taskalphabet}{\mymacro{T}}
\newcommand{\obtaskalphabet}{\mymacro{\overline{T}}}
\newcommand{\responseset}{\mymacro{R}}

\newcommand{\textexample}[1]{{\mymacro{ \fontfamily{ppl}\selectfont \text{``#1''}}}}

\DeclareMathOperator*{\argmax}{argmax}

\usepackage{cleveref}
\crefname{section}{\S}{\S\S}
\Crefname{section}{\S}{\S\S}
\crefname{table}{Tab.}{}
\crefname{figure}{Fig.}{}
\crefname{algorithm}{Alg.}{}
\crefname{appendix}{App.}{}
\crefname{lemma}{Lemma}{}
\Crefname{theorem}{Theorem}{}
\crefname{proposition}{Proposition}{}
\crefname{hypothesis}{Hypothesis}{}
\crefname{deduction}{Deduction}{}
\crefname{prediction}{Prediction}{Predictions}
\crefname{pred}{Prediction}{Predictions}
\crefname{cor}{Corollary}{}
\crefname{align}{}{}
\crefname{equation}{}{}

\usepackage{todonotes}
\usepackage{inconsolata}
\usepackage{multirow}

\usepackage{xcolor}
\definecolor{color1}{RGB}{102,194,165}
\definecolor{color2}{RGB}{141,160,203}
\definecolor{color3}{RGB}{252,141,98}
\definecolor{light-gray}{gray}{0.8}
\definecolor{colorbl}{RGB}{148,203,236}
\definecolor{colorbo}{RGB}{051,117,056}
\definecolor{colorsul}{RGB}{046,037,133}
\definecolor{colorsui}{RGB}{126,041,084}
\definecolor{colorc}{RGB}{221,221,221}
\definecolor{colorl}{RGB}{194,106,119}

\usepackage{tikz}
\tikzstyle{arrow} = [thick,->,>=stealth,line width=2pt]
\tikzstyle{arrowl} = [thick,->,>=stealth,line width=2mm,colorbl]
\tikzstyle{node1} = [rectangle, rounded corners, text centered, draw=black, minimum height=0.6cm, fill=colorc, text=colorl, font=\bfseries, thick]
\tikzstyle{node2} = [rectangle, rounded corners, text centered, draw=black, minimum height=0.6cm, fill=colorl, text=colorc, font=\bfseries, thick]

\newcommand{\rulesep}{\unskip\ \vrule\ }

\title{What Do Language Models Learn in Context? The Structured Task Hypothesis.}

\usepackage{amsmath}
\usepackage{amssymb}

\author{Jiaoda Li\thanks{\quad Equal contribution} \quad 
Yifan Hou$\footnotemark[1]$ \quad
Mrinmaya Sachan \quad 
Ryan Cotterell \quad 
\\
\setlength{\fboxsep}{2.5pt}%
\setlength{\fboxrule}{2.5pt}%
\fcolorbox{white}{white}{
    $\{$\texttt{\href{mailto:jiaoda.li@ai.ethz.ch}{jiaoda.li}, }
    \texttt{\href{mailto:yifan.hou@inf.ethz.ch}{yifan.hou}, }
    \texttt{\href{mailto:mrinmaya.sachan@inf.ethz.ch}{mrinmaya.sachan}, }
    \texttt{\href{mailto:ryan.cotterell@inf.ethz.ch}{ryan.cotterell}}$\}$\texttt{@inf.ethz.ch}
} \\
    {%
\setlength{\fboxsep}{2.5pt}%
\setlength{\fboxrule}{2.5pt}%
\fcolorbox{white}{white}{\includegraphics[width=.15\linewidth]{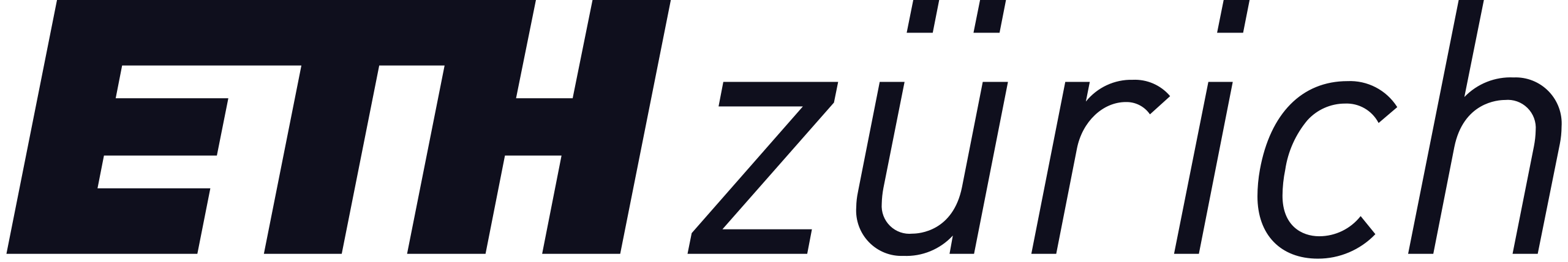}}
}
}

\begin{document}
\maketitle

\begin{abstract}
Large language models (LLMs) exhibit an intriguing ability to learn a novel task from in-context examples presented in a demonstration, termed in-context learning (ICL). 
Understandably, a swath of research has been dedicated to uncovering the theories underpinning ICL.
One popular hypothesis explains ICL by \textit{task selection}. 
LLMs identify the task based on the demonstration and generalize it to the prompt. 
Another popular hypothesis is that ICL is a form of \textit{meta-learning}, i.e., the models learn a learning algorithm at pre-training time and apply it to the demonstration.
Finally, a third hypothesis argues that LLMs use the demonstration to select a \textit{composition of tasks learned during pre-training} to perform ICL.
In this paper, we empirically explore these three hypotheses that explain LLMs' ability to learn in context with a suite of experiments derived from common text classification tasks. 
We invalidate the first two hypotheses with counterexamples
and provide evidence in support of the last hypothesis.
Our results suggest an LLM  could learn a novel task in context via composing tasks learned during pre-training.

\includegraphics[width=1.25em,height=1.15em]{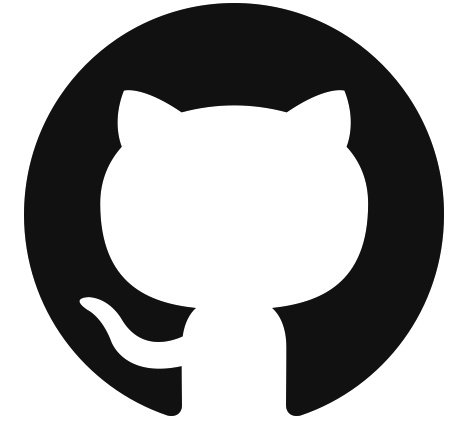}\hspace{.75em}
\parbox{\dimexpr\linewidth-7\fboxsep-7\fboxrule}{\url{https://github.com/eth-lre/LLM_ICL}}
\vspace{-.5em}
\end{abstract}

\section{Introduction}\label{sec:introduction}
In-context learning (ICL) is a learning paradigm where a pre-trained large language model (LLM) learns to perform a certain task by extrapolating beyond a demonstration of the task in the form of example prompt--response pairs, given to the model as input. In-context learning does \emph{not} require an update to the model's parameters \citep{radford2019language, NEURIPS2020_1457c0d6}. 
Conditioned on the demonstration, the LLM is then tasked with generating responses to additional related prompts.
Pre-trained large language models have exhibited an impressive ability to learn in context across various domains, e.g., code generation~\citep{DBLP:journals/corr/abs-2107-03374}, education~\citep{kasneci2023chatgpt}, and medicine~\citep{thirunavukarasu2023large}. 
However, there is still no consensus on when or how ICL works. 
We taxonomize existing candidate theories into three competing hypotheses (\cref{fig:motivation}), which we summarize below.\looseness=-1 

\begin{figure}
	\centering
	\includegraphics[width=.98\linewidth]{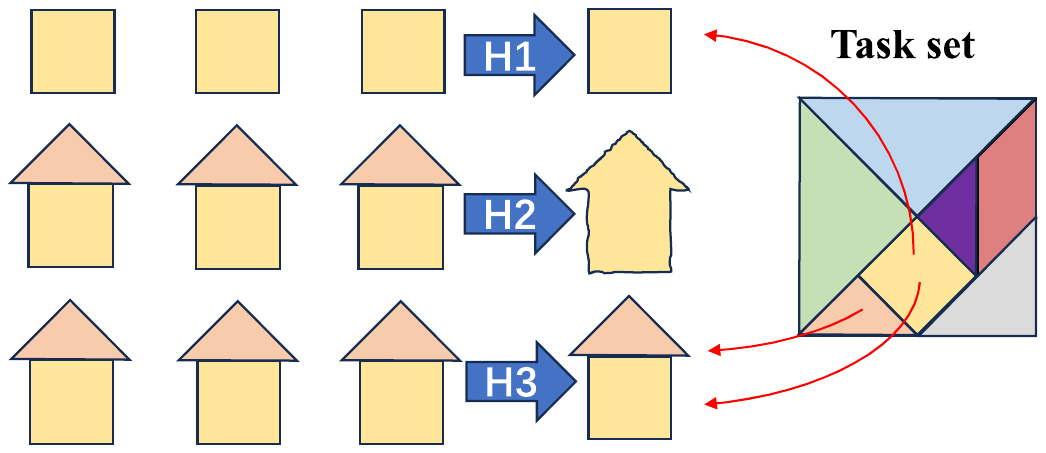}
    \vspace{-.2cm}
	\caption{The illustration of three hypotheses.}
	\label{fig:motivation}
    \vspace{-.2cm}
\end{figure}

\begin{hypothesis}[Informal; Task Selection]
    \label{hyp:recognize}
    During pre-training, an LLM learns a set of tasks $\obtasks$. 
    At inference time, the LLM identifies the task $\task\in\obtasks$ given the user-provided demonstration of the task and generalizes to the prompt.\looseness=-1
\end{hypothesis}
Under \cref{hyp:recognize}, the demonstration merely allows the model to \emph{recognize} a task, and no actual learning takes place. 
\citet{min-etal-2022-rethinking} offers empirical support for \cref{hyp:recognize}; they show that randomly shuffling the responses in the demonstration hardly has any effect on ICL performance, suggesting the demonstration of the task only serves to enable the LLM to look up a task.
In other words, \cref{hyp:recognize} asserts that no \emph{learning} takes place during ICL.
Some authors \citep{xie2022an, wang2023large, wies2023learnability} have also argued for \cref{hyp:recognize} from a theoretical angle, contending that if an LLM is pre-trained on a corpus that is generated from a mixture model over tasks, it will be able to infer the task that generated the demonstrations to be able to generalize to a prompt not in the demonstration.\looseness=-1

The next hypothesis revolves around meta-learning \citep{schmidhuber87, Thrun1998, vilalta01, pmlr-v70-finn17a}.
\begin{hypothesis}[Informal; Meta-Learning]
    \label{hyp:learn}
    During pre-training, an LLM learns certain learning algorithms. 
    During ICL, the LLM learns a task $\task$ directly from the demonstration using one of the learned learning algorithms.
\end{hypothesis}
\cref{hyp:learn} suggests that the pre-training stage prepares the parameters in an LLM in such a way that various learning algorithms, e.g., gradient descent and least squares regression, can be implicitly deployed during ICL \citep{pmlr-v202-von-oswald23a, akyurek2023what, dai-etal-2023-gpt} to learn a new task from a demonstration.
However, the setting assumed in the above-cited theoretical development is usually over-simplified, e.g., the assumption that the attention mechanism is linear.
Moreover, the empirical evidence for \cref{hyp:learn} is mostly derived from experiments on small transformers trained from scratch on synthetic data \citep{garg2022what, raventos2023pretraining, akyürek2024incontext}.\looseness=-1

\begin{figure*}
    \footnotesize
    \centering
    \begin{subfigure}[b]{.255\textwidth}
    \centering
    \begin{tikzpicture}
        \node (review) at (0, 0) [node1]  {cold movie};
        \node (response) [node2, right=1 of review]  {negative};
        \draw [arrow, colorbo] (review) -- node [above] {$\task$} (response) ;
    \end{tikzpicture}
    \caption{Task $\task$. The example prompt \textexample{cold movie} is paired with the response \textexample{negative}.}
    \label{fig:intro_icl}
    \end{subfigure}
    \rulesep
    \begin{subfigure}[b]{.33\textwidth}
    \centering
    \begin{tikzpicture}
        \node (review1) at (0, 0) [node1] {cold movie};
        \node (response1) [node2, right=1 of review1]  {bar};
        \node (review2) at (-0.7, -2) [node1] {cold movie};
        \node (intermediate) [node2, right=0.75 of review2] {negative};
        \node (response2) [node2, right=2.8 of review2] {bar};
        \draw [arrow, colorsul] (review1) -- node [above] {$\sul$} (response1);
        \draw [arrow, colorbo] (review2) -- node [above] {$\task$} (intermediate);
        \draw [arrow, colorsul] (intermediate) -- node [above] {$\g$} (response2);
        \draw [arrowl] (1, -1.4) -- (1, -0.4);
    \end{tikzpicture}
    \caption{RA task $\sul$. The original response (i.e., \textexample{negative}) is replaced with a random token (i.e., \textexample{bar}).}
    \label{fig:intro_sul}
    \end{subfigure}
    \rulesep
    \begin{subfigure}[b]{.395\textwidth}
    \centering
    \begin{tikzpicture}
        \node (review1) at (0, 0) [node1] {lorem ipsum};
        \node (response1) [node2, right=1 of review1]  {negative};
        \node (review2) at (-0.7, -2) [node1] {lorem ipsum};
        \node (intermediate) [node1, right=0.7 of review2] {cold movie};
        \node (response2) [node2, right=3 of review2] {negative};
        \draw [arrow, colorsui] (review1) -- node [above] {$\sui$} (response1);
        \draw [arrow, colorsui] (intermediate) -- node [above] {$\h$}  (review2);
        \draw [arrow, colorbo] (intermediate) -- node [above] {$\task$} (response2);
        \draw [arrowl] (1.3, -1.4) -- (1.3, -0.4);
    \end{tikzpicture}
    \caption{PA task $\sui$. The original prompt (i.e., \textexample{cold movie}) is transformed into random text (i.e., \textexample{lorem ipsum}).}
    \label{fig:intro_sui}
    \end{subfigure}
    \caption{Illustrations of a sentiment classification task, a response-altered (RA) task, and a prompt-altered (PA) task.\looseness=-1
    }
    \label{fig:intro}
    \vspace{-10pt}
\end{figure*}
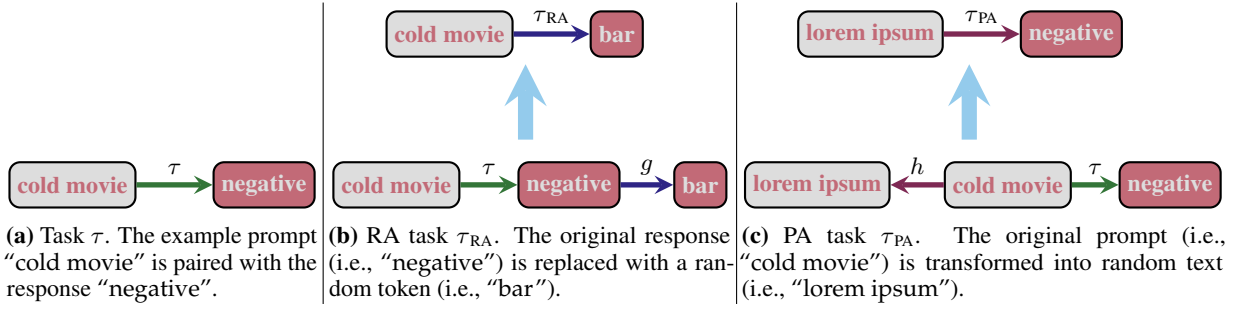

The final hypothesis may be viewed as a mixture of \cref{hyp:recognize} and \cref{hyp:learn}.
\begin{hypothesis}[Informal; Structured Task Selection]
    \label{hyp:compose}
    During pre-training, an LLM learns a set of tasks $\obtasks$.
    At inference time, the LLM uses the demonstration to compose a sequence of learned tasks $\task_1,\task_2,\ldots\in\obtasks$ and uses this composition for prediction.
    The composition itself may result in a novel task not seen during pre-training.
\end{hypothesis}
\cref{hyp:compose} extends \cref{hyp:recognize} by allowing ICL to not only index into the pre-learned task set $\obtasks$, but also compose tasks to obtain a novel task $\task=\task_1\circ\task_2\circ\cdots\notin \obtasks$.\footnote{The definition of task composition $\circ$ will be given in \cref{sec:formal_hyp}.} 
\citet{hahn2023theory} lay the theoretical groundwork for \cref{hyp:compose}; under their framework, they argue that the compositional structure of linguistic pre-training data gives rise to the task composition ability of ICL.\looseness=-1

In this paper, we conduct a comprehensive battery of experiments to examine the above-stated three hypotheses. 
Given a demonstration of a task, i.e., a sequence of prompt--response pairs from which an LLM can successfully learn the task in context,  
we create a response-altered (RA) task by altering the responses such that the new task is unlikely to have occurred during pre-training. 
Our results confirm that LLMs can learn such an RA task in context, which rejects \cref{hyp:recognize}. 
We then create prompt-altered (PA) tasks, which instead alter the prompts. 
If \cref{hyp:learn} is true, this modification should not affect the performance of ICL. 
However, we find that the LLMs yield a substantially worse performance on PA tasks than on RA tasks, which contradicts \cref{hyp:learn}.
Lastly, in support of \cref{hyp:compose}, we identify a sequence of simple tasks that the LLMs can compose to obtain unobserved RA tasks. It offers a possible explanation for ICL's ability to perform novel tasks.\looseness=-1

\section{Preliminaries}
\subsection{Language Models}
A language model $\lm$ is a distribution over $\alphabet^*$ where $\alphabet$ is an alphabet. 
The elements of $\alphabet$ are tokens.
A string $\bw=\w[1]\cdots\w[N]$ of length $N$ is a finite sequence of tokens $\w[n] \in \alphabet$.
Most modern LLMs are defined in an autoregressive manner, i.e.,\looseness=-1
\begin{equation}\label{eq:autoregressive}
    \lm(\bw) = \lm(\eos \mid \bw)\prod_{n=1}^{N}\lm(\w[n] \mid \bw[<n]),
\end{equation}
and each local conditional distribution $\lm(\cdot\mid \bw[<n])$ is defined over $\alphabeteos \defeq \alphabet \cup \{ \eos\}$.
We denote $\bw[<n] \defeq \w[1]\cdots\w[n-1]$ and $\bw[<1] \defeq \varepsilon$.
Note that not all models defined as in \cref{eq:autoregressive} are distributions over $\alphabet^*$.
However, in the context of our paper, we assume $\lm$ is.\looseness=-1

\subsection{A Restatement of the Hypotheses}
\label{sec:formal_hyp}
In this section, we offer a more formal version of each of the three hypotheses discussed in \cref{sec:introduction}.
Note that our treatment is decidedly \emph{not} a formalization.
Nevertheless, we do find it useful to build up some level of formal notation to discuss the three hypotheses more precisely; we save a true formalization for future work.\looseness=-1

We start with a concrete definition of a task. 
In this paper, a \defn{task} $\task \in \tasks$ is taken to be a pair of language models $\langle \promptlm_{\task}, \responselm_{\task} \rangle$ where both $\promptlm_{\task}$ and $\responselm_{\task}$ are distributions over $\alphabet^*$. 
We interpret $\promptlm_{\task}$ as a distribution over prompts for task $\task$ and $\responselm_{\task}$ as a distribution over responses \emph{conditioned} on a prompt.
Let $\tasks$ be a countable set of tasks.
A \defn{demonstration} of a task $\task \in \tasks$ is an interwoven sequence of prompt--response pairs.
We denote a demonstration of length $L$ as $\demonstration = \bp_1 \br_1\delim \cdots \delim\bp_L \br_L$, where $\delim \in \alphabet$ is a distinguished delimiter.
We assume $\bp_{\ell} \sim \promptlm_{\task}(\cdot)$ and $\br_{\ell} \sim \responselm_{\task}(\cdot \mid \bp_{\ell})$ for $1 \leq \ell \leq L$.
With $\obtasks \subset \tasks$, we denote the set of tasks observed at the pre-training time, i.e., those tasks where there exist demonstrations $\bp_1 \br_1\delim \cdots \delim\bp_L \br_L$ of the task in the pre-training data.

At inference time, given a demonstration $\demonstration = \bp_1 \br_1 \cdots \bp_L \br_L$, we say that a language model $\lm$ has (approximately) learned a task if 
\begin{equation}\label{eq:in-context-learning}
    \lm(\br \mid \demonstration \delim \bp) \approx \sum_{\task \in \tasks} \responselm_{\task} (\br \mid \bp) p(\task \mid \demonstration \delim \bp),
\end{equation}
and $p(\task \mid \demonstration \delim \bp)$ is sufficiently low entropy for large $L$.
In words, \cref{eq:in-context-learning} says that the task to be performed by the in-context learning may fruitfully be viewed as a latent variable \citep{xie2022an}.
And, moreover, when the task-selection distribution $p(\task \mid \demonstration \delim \bp)$ is low entropy for large enough $L$, i.e., when we observe a large enough demonstration, the sum is dominated by a single summand.
This means the language model has succeeded at identifying the task with high probability. 
\citet{xie2022an} present a more detailed theoretical framework to explore this scenario in a more precise manner.\looseness=-1

\setcounter{hypothesis}{0}
We now restate the hypotheses in our notation.\looseness=-1
\begin{hypothesis}[Task Selection]
    The task-selection distribution $p(\task \mid \demonstration\delim \bp)$ is only well-calibrated for tasks in $\obtasks$, i.e., the finite set of tasks observed at pre-training time.\looseness=-1
\end{hypothesis}

\begin{hypothesis}[Meta-Learning]
    The task-selection distribution $p(\task \mid \demonstration\delim \bp)$ generalizes to some tasks in $\tasks \setminus \obtasks$, i.e., tasks \emph{not} observed at pre-training time.\looseness=-1
\end{hypothesis}

To explain our third hypothesis, let $\taskalphabet$ be a set of \defn{primitive tasks}.
We define the notion of task composition as follows.
Given two tasks $\task_1 = \langle \responselm_{\task_1}, \promptlm_{\task_1} \rangle$ and $\task_2 = \langle \responselm_{\task_2}, \promptlm_{\task_2} \rangle$, we define
\begin{equation}
\task_1 \circ \task_2 \defeq \langle \responselm_{\task_1 \circ \task_2}, \promptlm_{\task_2} \rangle,
\end{equation}
where we further define
\begin{equation}
    \responselm_{\task_1 \circ \task_2} (\br \mid \bp) \defeq \sum_{\widetilde{\br} \in \alphabet^*} \responselm_{\task_1}(\br \mid \widetilde{\br}) \responselm_{\task_2} (\widetilde{\br}  \mid \bp), 
\end{equation}
It is easy to see that $\circ$ is associative.\footnote{See \cref{app:proof}.}
Then, consider the semigroup\footnote{A semigroup is a set endowed with an associative operator, under which the set is closed.} $(\taskalphabet^*, \circ)$ where $\circ$ is as defined above.
In other words, any composition of primitive tasks $\task_1 \circ \cdots \circ \task_K \in \taskalphabet^*$ results in a new task.
This semigroup structure encodes a primitive (non-hierarchical) notion of task composition. 
We then take $\tasks = \taskalphabet^*$.
Finally, let $\obtaskalphabet$ be the set of observed primitive tasks. 
Consider $(\obtaskalphabet^*,\circ)$, a subsemigroup of $(\taskalphabet^*,\circ,)$.
Note that $\obtaskalphabet^*$ is larger than $\obtasks$, as it includes tasks \emph{not} observed during the pre-training time, but whose composite primitive tasks were. 
With this notation, we now present the third hypothesis.\looseness=-1 

\begin{hypothesis}[Structured Task Selection]
    The task-selection distribution $p(\task \mid \demonstration)$ is only well-calibrated for $(\obtaskalphabet^*,\circ)$, i.e., compositions of primitive tasks observed at pre-training time. 
\end{hypothesis}

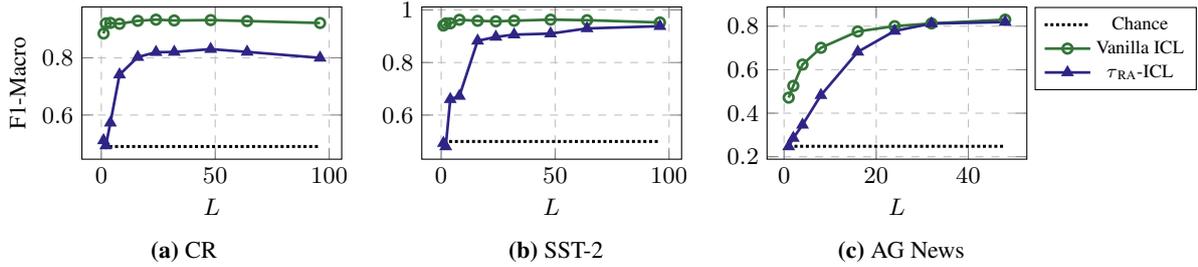
\begin{figure*}[!htbp]
\raggedright
\begin{subfigure}{.3\textwidth}
    \raggedright
    \begin{tikzpicture}
    \scalefont{0.8}
    \begin{axis}[
    sharp plot,
    xlabel=$L$,
    ylabel=F1-Macro,
    width=5cm, height=3.6cm,
    xlabel near ticks,
    ylabel near ticks,
    xmajorgrids=true,
    ymajorgrids=true,
    grid style=dashed,
    legend style={at={(0.9,1.1)}, anchor=south},
    legend columns=3, %
    legend pos=outer north east, %
    legend style={nodes={scale=0.8, transform shape}},]
    \addplot+[line width=0.4mm, densely dotted, mark=., color=black,] plot coordinates { 
        (1, 0.4898040493)
        (96, 0.4898040493)
    };
    \addplot+[line width=0.36mm, mark=o, mark options={scale=.9}, color=colorbo] plot coordinates {(1, 0.8851730964) (2, 0.9205265309) (4, 0.9223953824) (8, 0.9192026189) (16, 0.9295924492) (24, 0.9337266643) (32, 0.9308682038) (48, 0.932161613) (64, 0.9288226723) (96, 0.9220212383)};%
    \addplot+[line width=0.36mm, mark=triangle*, mark options={scale=.9}, color=colorsul] plot coordinates {(1, 0.5109917704) (2, 0.4932524376) (4, 0.5727940949) (8, 0.741265894) (16, 0.8026284498) (24, 0.8198860679) (32, 0.8205913419) (48, 0.8305165683) (64, 0.8208618017) (96, 0.8000396788)};%
    \end{axis}
    \end{tikzpicture}
    \caption{CR}
    \label{fig:icl_num_cr}
\end{subfigure}
\begin{subfigure}{.3\textwidth}
    \raggedright
    \begin{tikzpicture}
    \scalefont{0.8}
    \begin{axis}[
    sharp plot,
    xlabel=$L$,
    width=5cm, height=3.6cm,
    xlabel near ticks,
    ylabel near ticks,
    xmajorgrids=true,
    ymajorgrids=true,
    grid style=dashed,
    legend style={at={(0.9,1.1)}, anchor=south},
    legend columns=3, %
    legend pos=north west, %
    legend style={nodes={scale=0.8, transform shape}},]
    \addplot+[line width=0.4mm, densely dotted, mark=., color=black,] plot coordinates { 
        (1, 0.4995513957)
        (96, 0.4995513957)
    };
    \addplot+[line width=0.36mm, mark=o, mark options={scale=.9}, color=colorbo] plot coordinates {(1, 0.9408225998) (2, 0.9484771912) (4, 0.948958084) (8, 0.9626997295) (16, 0.9587636902) (24, 0.9568076034) (32, 0.9591637387) (48, 0.9631488021) (64, 0.9611319219) (96, 0.9524214155)};%
    \addplot+[line width=0.36mm, mark=triangle*, mark options={scale=.9}, color=colorsul] plot coordinates {(1, 0.4931789017) (2, 0.4803575639) (4, 0.6600471199) (8, 0.6724065982) (16, 0.8825932413) (24, 0.897344605) (32, 0.9061587534) (48, 0.9100169232) (64, 0.929713742) (96, 0.9382894455)};%
    \end{axis}
    \end{tikzpicture}
    \caption{SST-2}
    \label{fig:icl_num_sst2}
\end{subfigure}
\hspace{-.45cm}
\begin{subfigure}{.3\textwidth}
    \raggedright
    \begin{tikzpicture}
    \scalefont{0.8}
    \begin{axis}[
    sharp plot,
    xlabel=$L$,
    width=5cm, height=3.6cm,
    xlabel near ticks,
    ylabel near ticks,
    xmajorgrids=true,
    ymajorgrids=true,
    grid style=dashed,
    legend style={at={(0.9,1.1)}, anchor=south},
    legend columns=1, %
    legend pos=outer north east, %
    legend style={nodes={scale=0.8, transform shape}},]
    \addplot+[line width=0.4mm, densely dotted, mark=., color=black,] plot coordinates { 
        (1, 0.2477786384)
        (48, 0.2477786384)
    };
    \addlegendentry{Chance}
    \addplot+[line width=0.36mm, mark=o, mark options={scale=.9}, color=colorbo] plot coordinates {(1, 0.4715167392) (2, 0.5255636356) (4, 0.623556005) (8, 0.7007062289) (16, 0.7758450964) (24, 0.8000677889) (32, 0.8125316722) (48, 0.830045702)};%
    \addlegendentry{Vanilla ICL} 
    \addplot+[line width=0.36mm, mark=triangle*, mark options={scale=.9}, color=colorsul] plot coordinates {(1, 0.2468478101) (2, 0.2856521754) (4, 0.3462552517) (8, 0.4825528712) (16, 0.6819940028) (24, 0.7780325244) (32, 0.8120065361) (48, 0.8190959566)};%
    \addlegendentry{$\sul$-ICL} 
    \end{axis}
    \end{tikzpicture}
    \caption{AG News}
    \label{fig:icl_num_agn}
\end{subfigure}
\caption{Performance of \textcolor{colorbo}{vanilla ICL} and \textcolor{colorsul}{$\sul$-ICL} on the $3$ datasets with different demonstration lengths $L$. LLaMA2-70B is used. The LLM is able to learn \textcolor{colorsul}{RA} tasks as $L$ grows.}
\label{fig:icl_num_sul}
\end{figure*}

\section{Testing \cref{hyp:recognize}}
\label{sec:h1}
According to \cref{hyp:recognize}, ICL selects the task it needs to perform from the demonstration.
However, it is only able to select among the finite set of tasks observed at the pre-training time, denoted as $\obtasks$.
It follows from \cref{hyp:recognize}, then, that if a novel task that has never been seen during pre-training is presented to a pre-trained model as a demonstration, ICL should not be able to perform it. 
We construct such a novel task as follows. 
We first create a string-to-string function $\g\colon \alphabet^* \rightarrow \alphabet^*$.
Note that such a function $\g$ is a special case of a response distribution that places probability 1 on a specific output string for every input string.
Then, given a task  $\langle \promptlm_{\task}, \responselm_{\task} \rangle$, we obtain a new task $\sul$ by applying $\g$ to the responses of $\task$, i.e.,
\begin{subequations}
\label{eq:sul}
\begin{align}
    \promptlm_{\sul}(\bp) &\defeq \promptlm_{\task}\left(\bp\right) \\
    \responselm_{\sul}(\br \mid \bp) &\defeq \responselm_{\task}\left(\g^{-1}(\br) \mid \bp\right).
\end{align}
\end{subequations}
Note that the definition in \cref{eq:sul} is no more than task composition, i.e., $\langle \g, \bullet \rangle \circ \langle \responselm, \promptlm \rangle$ where $\bullet$ is a stand-in for an arbitrary prompt distribution. 
We define $\taskg \defeq \langle \g, \bullet \rangle$ for the remainder of the paper, i.e., a task induced by the string-to-string function $\g$ with a stand-in prompt distribution; we also write $\sul =  \taskg \circ \task$.
The 
new task $\sul$ is almost certainly not observed in the pre-training data, i.e., $\sul \notin \obtasks$.
We call $\sul$ a response-altered task (\defn{RA}) and the ICL setting with RA tasks \defn{$\sul$-ICL}. 
This setting resembles the semantically unrelated label ICL setting of \citet{wei2023larger} and the abstract formalization of \citet{pan-etal-2023-context}.\looseness=-1

We examine the following logical consequence of \cref{hyp:recognize}.\looseness=-1
\begin{prediction}
\label{ded:recognize}
    If \cref{hyp:recognize} is true, then an LLM's performance in the $\sul$-ICL setting should be similar to random guessing.
\end{prediction}

\subsection{Experimental Setup}
\label{sec:h1_exp}
\paragraph{Tasks.} As shown in \cref{tab:nlp_datasets} (\cref{appendix:template}), we select $3$ commonly used text classification datasets for ICL: Customer Reviews \citep[CR;][]{DBLP:conf/kdd/HuL04}, Standford Sentiment Treebank with binary sentiments \citep[SST-2;][]{socher-etal-2013-recursive}, and AG News \citep{DBLP:conf/nips/ZhangZL15}; see \cref{appendix:license} for the license details.\looseness=-1 

\paragraph{Experimental Setup.} 
Each text classification dataset contains a set of pairs $\{(\bx^{(\ell)},\y^{(\ell)})\}_{\ell=1}^L$ where $\bx^{(\ell)}\in\alphabet^*$ is an input string and $\y^{(\ell)} \in \Y$ is $\bx^{(\ell)}$'s classification label drawn from $\Y$, a finite, task-dependent label set. 
To encode a classification problem as ICL, we map each element of $\Y$ to a string in $\alphabet^*$ by means of a function
$\classtoresponse \colon \Y \rightarrow \alphabet^*$. 
Additionally, we convert the input $\bx$ into a prompt $\bp$ through a templating function $\template\colon \alphabet^*\rightarrow \alphabet^*$. 
The exact templating function we use for each dataset is listed in \cref{tab:nlp_datasets}. 
Then, we construct a delimiter-separated demonstration of size $L$:
$\demonstration=\template(\bx^{(1)})\classtoresponse(\y^{(1)})\delim \cdots \delim \template(\bx^{(L)}) \classtoresponse(\y^{(L)})$.
To perform classification on a test prompt $\template(\bx)$, we select the highest-probability class as follows
\begin{equation}
   \y^\star =  \argmax_{\widetilde{\y}\in\Y} \lm(\classtoresponse(\widetilde{\y}) \mid \demonstration \delim \template(\bx)).
\end{equation}

\paragraph{Settings.} We consider three settings: (1) Chance: 
random guessing uniformly across different classes;\footnote{All the datasets are largely class-balanced.} (2) \textcolor{colorbo}{Vanilla ICL}: the standard ICL setting; (3) \textcolor{colorsul}{$\sul$-ICL}: The responses $\br$ in the demonstrations are replaced by $\g(\br)$. The prompts $\bp$ are left unchanged.\looseness=-1

\paragraph{Implementation Details.} 
We conduct experiments on a publicly available LLM: LLaMA2 \citep{touvron2023llama} with three sizes: 7B, 13B, and 70B. Each experiment is repeated $20$ times with different random seeds and the average F1-Macro score is reported. 
In each experiment, we construct a distinct test set consisting of $256$ prompt--response pairs, and for each element of this test set, we sample $L$
prompt--response pairs from the training set, which serve as its demonstration.\looseness=-1

\subsection{Results}
\label{sec:h1_res}
\paragraph{ICL Example Number.} We present results on LLaMA2-70B with varying demonstration lengths (labeled as $L$) in \cref{fig:icl_num_sul}. 
The $\sul$-ICL's performance is near chance when the number of prompt--response pairs in the demonstration is small, but it quickly grows to above $80\%$ as $L$ increases and matches the performance of vanilla ICL on SST-2 and AG News when $L$ is large.\looseness=-1

\paragraph{Model Size.} 
In \cref{fig:icl_modelsize_sul}, we report the ICL's performance ($L=32$) of models of different sizes. 
The performance of both vanilla ICL and $\sul$-ICL generally improves as the model size increases, but even the smallest model (LLaMA2-7B) yields a performance well above chance in $\sul$-ICL setting.\looseness=-1

\paragraph{Summary.} 
These results contradict \cref{ded:recognize} and demonstrate that an LLM \emph{can} learn RA tasks $\sul$ in context, which are highly unlikely to belong to the set of observed tasks $\obtasks$ during pre-training.
These experiments speak against \cref{hyp:recognize}.\looseness=-1

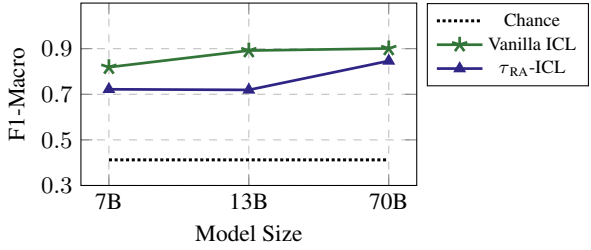
\begin{figure}[!t]
    \centering
    \scalefont{0.8}
    \begin{tikzpicture}
        \begin{axis}[
        sharp plot,
        xlabel=Model Size,
        ylabel=F1-Macro,
        width=6cm, height=4cm,
        ymin=0.3, ymax=1.1,
        xtick={7B, 13B, 70B}, 
        symbolic x coords={7B, 13B, 70B}, 
        ytick={0.30, 0.5, 0.7, 0.9},
        xlabel near ticks,
        ylabel near ticks,
        xmajorgrids=true,
        ymajorgrids=true,
        grid style=dashed,
        legend style={at={(0.9,1.1)}, anchor=south},
        legend columns=1, %
        legend pos=outer north east, %
        legend style={nodes={scale=0.8, transform shape}},
        ]
        \addplot+[line width=0.4mm, densely dotted, mark=., color=black,] plot coordinates { 
            (7B, 0.4123780278)
            (13B, 0.4123780278)
            (70B, 0.4123780278)
        };
        \addlegendentry{Chance}
        \addplot+[line width=0.36mm, mark=star, mark options={scale=1.5}, color=colorbo] plot coordinates {
            (7B, 0.8188092389)
            (13B, 0.8917246237)
            (70B, 0.9008545382)
        };
        \addlegendentry{Vanilla ICL} 
        \addplot+[line width=0.36mm, mark=triangle*, mark options={scale=.9}, color=colorsul] plot coordinates {
            (7B, 0.7218693832)
            (13B, 0.7191964815)
            (70B, 0.8462522105)
        };
        \addlegendentry{$\sul$-ICL} 
        \end{axis}
    \end{tikzpicture}
    \caption{Average performance of \textcolor{colorbo}{vanilla ICl} and \textcolor{colorsul}{$\sul$-ICL} across $3$ datasets (CR, SST-2, AG News). Demonstration length $L=32$. LLaMA2-70B yields the best performance but LLaMA2-7B is not far behind.}
    \label{fig:icl_modelsize_sul}
   \vspace{-5pt}
\end{figure}

\section{Testing \cref{hyp:learn}}
\setcounter{theorem}{2}
We have shown in the previous section that an LLM can learn a novel task in context, but it is still unclear \emph{what type} of tasks can be learned in context. 
If we accept \cref{hyp:learn} as true, certain learning algorithms are learned by the pre-trained LLM, so learning from a demonstration in context happens on the fly at inference time. 
This hypothesis implies that there need not be knowledge of the prompted task in the pre-training data. 
And indeed, on this view, the role of the pre-training is merely to prepare the parameters in the LLM in such a way that the LLM architecture encodes various learning algorithms.
For instance, some authors \citep{pmlr-v202-von-oswald23a, akyurek2023what, dai-etal-2023-gpt} argue that training a linear model with gradient descent can be encoded as in-context learning under certain simplifying assumptions.
This leads to the following prediction.
\begin{prediction}
\label{pred:learn_1}
    If \cref{hyp:learn} is true, then ICL should behave similarly to a model trained with a certain learning algorithm, e.g., gradient descent.
\end{prediction}
A caveat of \cref{pred:learn_1} is, of course, that we do not know \textit{a priori} which learning algorithm the LLMs learn at pre-training time.

\subsection{Experiment 1}
\label{sec:h2_1}
We first propose a prompt-altered (PA) task, where instead of transforming the responses as $\sul$ does, we transform the prompts, i.e., we create a string-to-string function $\h\colon \alphabet^* \rightarrow \alphabet^*$
and apply it to the prompts of $\task$.
This results in the following novel task $\sui = \langle \promptlm_{\sui}, \responselm_{\sui}\rangle$ where distributions are defined as
\begin{subequations}
\label{eq:sui}
\begin{align}
\promptlm_{\sui}(\bp) &=\promptlm_\task(\h^{-1}(\bp)) \\
\responselm_{\sui}(\br \mid \bp) &=\responselm_\task(\br\mid \h^{-1}(\bp)).
\end{align}
\end{subequations}
The ICL setting with PA tasks is named \defn{$\sui$-ICL}. We choose a $\h$ such that the PA task can be learned given the demonstration. 
If the RA task $\sul$ in the $\sul$-ICL setting is indeed learned through a meta-learning-esque procedure, $\sui$-ICL should be just as easily learnable in the $\sui$-ICL setting because $\sul$ and $\sui$ are essentially the same task. 
And, moreover, performing both tasks in context should exhibit similar performance to a logistic regression classifier if the LLM encodes the ability to learn a linear model implicitly in its parameters.

\begin{figure}
    \centering
    \begin{tikzpicture}
    \scalefont{0.8}
    \begin{axis}[
    sharp plot,
    xlabel=$L$,
    ylabel=F1-Macro,
    width=6cm, height=4cm,
    ymin=0.2, ymax=1.,
    xlabel near ticks,
    ylabel near ticks,
    xmajorgrids=true,
    ymajorgrids=true,
    grid style=dashed,
    legend style={at={(0.9,1.1)}, anchor=south},
    legend columns=1, %
    legend pos=outer north east, %
    legend style={nodes={scale=0.8, transform shape}},]
    \addplot+[line width=0.4mm, densely dotted, mark=., color=black,] plot coordinates { 
        (0, 0.4123780278)
        (48, 0.4123780278)
    };
    \addlegendentry{Chance}
    \addplot+[line width=0.36mm, mark=o, mark options={scale=.9}, color=colorbo] plot coordinates {(1, 0.7658374785) (2,0.7981891192) (4, 0.8316364905) (8, 0.8608695258) (16, 0.8880670786) (24, 0.8968673522) (32, 0.9008545382) (48,0.908452039)};%
    \addlegendentry{Vanilla ICL} 
    \addplot+[line width=0.36mm, mark=triangle*, mark options={scale=.9}, color=colorsul] plot coordinates {(1, 0.4170061607) (2, 0.419754059) (4, 0.5263654888) (8, 0.6320751211) (16, 0.7890718979) (24, 0.8317543991) (32, 0.8462522105) (48, 0.853209816)};%
    \addlegendentry{$\sul$-ICL} 
    \addplot+[line width=0.36mm, mark=square*, mark options={scale=.7}, color=colorsui] plot coordinates {(1, 0.4202502765) (2, 0.3819114795) (4, 0.3831205419) (8, 0.3639630335) (16, 0.3733709004) (24, 0.3814661474) (32, 0.3560990056) (48, 0.3562824624)};
    \addlegendentry{$\sui$-ICL} 
    \addplot+[line width=0.36mm, mark=star, mark options={scale=1.5}, color=colorbl] plot coordinates {(4, 0.4278610283) (8,0.4236840805) (16,0.4651727911) (24,0.4886777723) (32,0.5050218925) (48,0.5277380176)};
    \addlegendentry{$\sui$-LR}
    \end{axis}
    \end{tikzpicture}
    \caption{Performance of various settings across $3$ text classification tasks. LLaMA2-70B is used. \textcolor{colorsui}{$\sui$-ICL} performs worse than \textcolor{colorbl}{$\sui$-LR} and chance.\looseness=-1 
    }
\label{fig:icl_numdrop}
\end{figure}
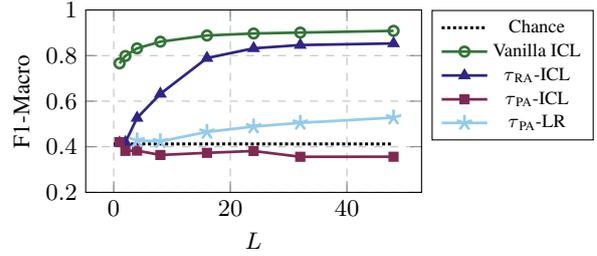

\subsubsection{Experimental Setup} 
Our experimental setup follows \cref{sec:h1_exp}.
However, we introduce two more settings in addition to the ones we consider in \cref{sec:h1_exp}:
\begin{itemize}[leftmargin=*]
    \item \textbf{\textcolor{colorsui}{$\sui$-ICL}}: As defined in \cref{eq:sui}, the prompts $\bp$ (including the template tokens such as \textexample{Review} and \textexample{Sentiment}) in the demonstrations are replaced with $\h(\bp)$. %
    The responses $\br$ are left unchanged.\looseness=-1
    \item \textbf{\textcolor{colorbl}{$\sui$-LR}}: 
    To make sure a PA task $\sui$ is learnable from a demonstration, we fit a logistic regression (LR) model as a baseline.
    Because a variable-length string cannot be easily fed into a logistic regressor, we first tokenize a string using the LLaMA2 tokenizer and then convert it into bag-of-words (BoW) representations.\footnote{More accurately, a bag of tokens.} 
    The classifier is then trained using a BoW representation of exactly the same $L$ prompt--response pairs as used in a demonstration of $\sui$-ICL. 
    We use this baseline to gauge the performance of a model with minimum learning ability.
\end{itemize}

\begin{table*}
    \small
    \centering
    \begin{tabular}{ccccc}
    \toprule
    \midrule
        {Dataset} & $\taskg$-Linear (F1-Macro $\%$) & $\taskg$-ICL (F1-Macro $\%$) & $\pearson$  & $\spearman$ \\
        \midrule
         CR / SST-2 & {$100.0\pm 0.0$} &{$92.7\pm 15.2$} & {N.A.} & {N.A.}  \\
         AG News & $100.0\pm 0.0$ & $93.8\pm 10.8$ & N.A. & N.A.  \\
         DBPedia & $99.9\pm 0.8$ & $59.8\pm 19.7$ & $-0.02$ ($0.59$) & $0.01$ ($0.83$)  \\
        \midrule
        \bottomrule
    \end{tabular}
    \caption{Means and variances of the performance of $\taskg$-Linear and $\taskg$-ICL. $\pearson$ is the Pearson correlation coefficient between $\taskg$-Linear and $\taskg$-ICL and $\spearman$ is the Spearman's rank correlation coefficient. In the parentheses are $p$-values. No significant correlation is observed.\looseness=-1 }
    \label{tab:linear_corr}
\end{table*}

\subsubsection{Results}
The average accuracy of LLaMA2-70B across $3$ datasets are in \cref{fig:icl_numdrop}. 
More detailed results can be found in \cref{appendix:icl_num}.\looseness=-1

\paragraph{$\sui$-LR.} 
We find that the logistic regressor is able to learn the tasks to some degree, achieving an F1-Macro score of around $50\%$.
In contrast, a random guesser achieves an F1-Macro score of $41\%$.
This experiment shows the task is indeed learnable to a certain extent given the demonstration as training data. 
And, as expected, the performance improves steadily as $L$ increases. 

\paragraph{$\sui$-ICL.} 
However, in the $\sui$-ICL setting, the LLM always performs near or below chance regardless of the size of the demonstration.
Indeed, the performance does not improve even when the demonstration length $L$ reaches the maximum number of tokens allowed for LLaMA2 (\cref{fig:icl_num_all}).\looseness=-1

\paragraph{$\sul$-ICL.} 
In stark contrast to the $\sui$-ICL setting, as shown in \cref{fig:icl_numdrop}, the LLM has an average score above $80\%$ in the $\sul$-ICL setting. 

\paragraph{Summary.} The enormous performance gap between $\sul$-ICL and $\sui$-ICL does not concord with \cref{pred:learn_1}, and, thereby gives us evidence against \cref{hyp:learn}. 
Specifically, our results imply that even though the RA tasks $\sul$ are novel, they are not learned with some learning algorithm on the fly at inference time. This is in line with the observations of \citet{kossen2024incontext} and \citet{shen2024pretrained}.\looseness=-1

\subsection{Experiment 2}
\label{sec:h2_2}
In our second experiment, we focus on one specific theoretical claim---specifically, that LLMs may implicitly learn a linear regression using gradient descent during ICL \citep{pmlr-v202-von-oswald23a, akyurek2023what, dai-etal-2023-gpt}.\looseness=-1

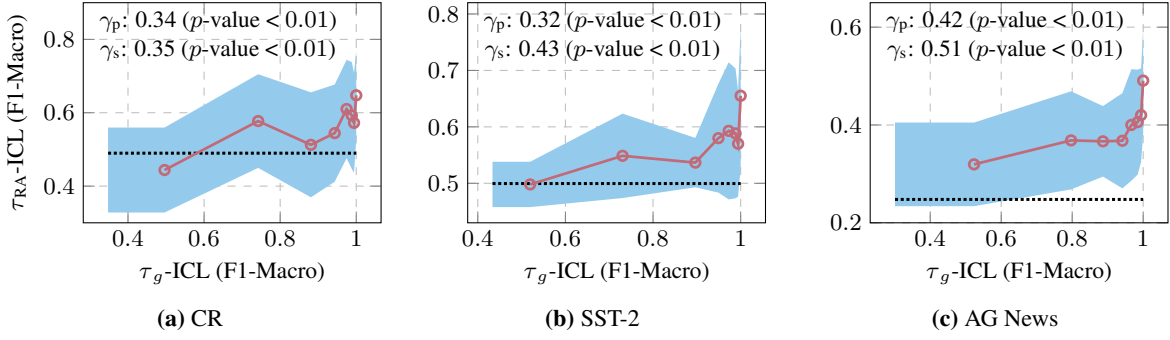
\begin{figure*}
\centering
\begin{subfigure}{.325\textwidth}
    \centering
    \begin{tikzpicture}
    \scalefont{0.8}
    \begin{axis}[
    sharp plot,
    xlabel=$\taskg$-ICL (F1-Macro),
    ylabel=$\sul$-ICL (F1-Macro),
    width=5.5cm, height=4.5cm,
    ymin=0.3, ymax=0.9,
    xlabel near ticks,
    ylabel near ticks,
    xmajorgrids=true,
    ymajorgrids=true,
    grid style=dashed,
    legend style={at={(0.9,1.1)}, anchor=south},
    legend columns=1, %
    legend pos=north east, %
    legend style={nodes={scale=0.8, transform shape}},]
    \node[below right, align=center, text=black]
        at (rel axis cs:0., 1.) {\,\,$\pearson$: 0.34 ($p$-value < $0.01$) \\ 
        \,\,$\spearman$: 0.35 ($p$-value < $0.01$)};
    \addplot+[line width=0.4mm, densely dotted, mark=., color=black,] plot coordinates { 
        (0.34765625, 0.489804049330865)
        (1, 0.489804049330865)
    };
    \addplot+[line width=0.36mm, mark=o, mark options={scale=.9}, color=colorl] plot coordinates {
        ( 0.49609375 , 0.44386574074074076 )
        ( 0.7421875 , 0.57734375 )
        ( 0.880859375 , 0.5125269396551724 )
        ( 0.943359375 , 0.5445033482142857 )
        ( 0.974609375 , 0.6101262019230769 )
        ( 0.98828125 , 0.5926106770833334 )
        ( 0.994140625 , 0.5724158653846154 )
        (1, 0.6475797086148649 )};
    \addplot [name path=upper,draw=none] coordinates {
        ( 0.34765625 , 0.32808937903468643 )
        ( 0.49609375 , 0.32808937903468643 )
        ( 0.7421875 , 0.4497406462629305 )
        ( 0.880859375 , 0.3694303228141154 )
        ( 0.943359375 , 0.41205514080310723 )
        ( 0.974609375 , 0.47592015723451897 )
        ( 0.98828125 , 0.44705162410068433 )
        ( 0.994140625 , 0.4363760385217408 )
        (1, 0.5262459454210734 )};
    \addplot [name path=lower,draw=none] coordinates {
        ( 0.34765625 , 0.559642102446795 )
        ( 0.49609375 , 0.559642102446795 )
        ( 0.7421875 , 0.7049468537370696 )
        ( 0.880859375 , 0.6556235564962294 )
        ( 0.943359375 , 0.6769515556254642 )
        ( 0.974609375 , 0.7443322466116348 )
        ( 0.98828125 , 0.7381697300659824 )
        ( 0.994140625 , 0.70845569224749 )
        (1, 0.7689134718086564 )};
    \addplot [fill=colorbl] fill between[of=upper and lower];
    \end{axis}
    \end{tikzpicture}
    \caption{CR}
    \label{fig:corr_all_cr}
\end{subfigure}
\begin{subfigure}{.325\textwidth}
    \centering
    \begin{tikzpicture}
    \scalefont{0.8}
    \begin{axis}[
    sharp plot,
    xlabel=$\taskg$-ICL (F1-Macro),
    width=5.5cm, height=4.5cm,
    ymin=0.43, ymax=.82,
    xlabel near ticks,
    ylabel near ticks,
    xmajorgrids=true,
    ymajorgrids=true,
    grid style=dashed,
    legend style={at={(0.9,1.1)}, anchor=south},
    legend columns=1, %
    legend pos=north east, %
    legend style={nodes={scale=0.8, transform shape}},]
    \node[below right, align=center, text=black]
        at (rel axis cs:0., 1.) {\,\,$\pearson$: 0.32 ($p$-value < $0.01$) \\ 
        \,\,$\spearman$: 0.43 ($p$-value < $0.01$)};
    \addplot+[line width=0.4mm, densely dotted, mark=., color=black,] plot coordinates { 
        (0.43359375, 0.49955139565955)
        (1, 0.49955139565955)
    };
    \addplot+[line width=0.36mm, mark=o, mark options={scale=.9}, color=colorl] plot coordinates {
        ( 0.51953125, 0.49797453703703703 )
        ( 0.73046875, 0.5486886160714286 )
        ( 0.896484375, 0.5366586538461539 )
        ( 0.94921875, 0.5800057870370371 )
        ( 0.97265625, 0.5929418103448276 )
        ( 0.98828125, 0.58837890625 )
        ( 0.994140625, 0.5701069078947368 )
        (1, 0.6550897277227723 )};
    \addplot [name path=upper,draw=none] coordinates {
        ( 0.43359375, 0.45782234083633877 )
        ( 0.51953125, 0.45782234083633877 )
        ( 0.73046875, 0.4738958249821679 )
        ( 0.896484375, 0.4929307001606176 )
        ( 0.94921875, 0.4834736552018936 )
        ( 0.97265625, 0.47144516018631544 )
        ( 0.98828125, 0.4734936847439969 )
        ( 0.994140625, 0.47636141660195486 )
        (1, 0.5287618099090919 )};
    \addplot [name path=lower,draw=none] coordinates {
        ( 0.43359375, 0.5381267332377353 )
        ( 0.51953125, 0.5381267332377353 )
        ( 0.73046875, 0.6234814071606893 )
        ( 0.896484375, 0.5803866075316901 )
        ( 0.94921875, 0.6765379188721805 )
        ( 0.97265625, 0.7144384605033398 )
        ( 0.98828125, 0.7032641277560031 )
        ( 0.994140625, 0.6638523991875188 )
        (1, 0.7814176455364527 )};
    \addplot [fill=colorbl] fill between[of=upper and lower];
    \end{axis}
    \end{tikzpicture}
    \caption{SST-2}
    \label{fig:corr_all_sst2}
\end{subfigure}
\begin{subfigure}{.325\textwidth}
    \centering
    \begin{tikzpicture}
    \scalefont{0.8}
    \begin{axis}[
    sharp plot,
    xlabel=$\taskg$-ICL (F1-Macro),
    width=5.5cm, height=4.5cm,
    ymin=0.2, ymax=.65,
    xlabel near ticks,
    ylabel near ticks,
    xmajorgrids=true,
    ymajorgrids=true,
    grid style=dashed,
    legend style={at={(0.9,1.1)}, anchor=south},
    legend columns=1, %
    legend pos=north east, %
    legend style={nodes={scale=0.8, transform shape}},]
    \node[below right, align=center, text=black]
        at (rel axis cs:0., 1.) {\,\,$\pearson$: 0.42 ($p$-value < $0.01$) \\ 
        \,\,$\spearman$: 0.51 ($p$-value < $0.01$)};
    \addplot+[line width=0.4mm, densely dotted, mark=., color=black,] plot coordinates { 
        (0.30078125, 0.247778638411707)
        (1, 0.247778638411707)
    };
    \addplot+[line width=0.36mm, mark=o, mark options={scale=.9}, color=colorl] plot coordinates {
        ( 0.5234375, 0.31946790540540543 )
        ( 0.796875, 0.3685238486842105 )
        ( 0.88671875, 0.3665707236842105 )
        ( 0.94140625, 0.3677455357142857 )
        ( 0.966796875, 0.4000459558823529 )
        ( 0.984375, 0.4058159722222222 )
        ( 0.994140625, 0.42003676470588236 )
        (1, 0.49070581896551724 )};
    \addplot [name path=upper,draw=none] coordinates {
        ( 0.30078125, 0.23399650771601732 )
        ( 0.5234375, 0.23399650771601732 )
        ( 0.796875, 0.26822940079984653 )
        ( 0.88671875, 0.2944899439309495 )
        ( 0.94140625, 0.27043972315652787 )
        ( 0.966796875, 0.28602187310054167 )
        ( 0.984375, 0.29849563756673764 )
        ( 0.994140625, 0.32423059787242825 )
        (1, 0.3823571096750702 )};
    \addplot [name path=lower,draw=none] coordinates {
        ( 0.30078125 , 0.40493930309479353 )
        ( 0.5234375 , 0.40493930309479353 )
        ( 0.796875 , 0.4688182965685745 )
        ( 0.88671875 , 0.4386515034374715 )
        ( 0.94140625 , 0.46505134827204353 )
        ( 0.966796875 , 0.5140700386641641 )
        ( 0.984375 , 0.5131363068777067 )
        ( 0.994140625 , 0.5158429315393365 )
        (1, 0.5990545282559643 )};
    \addplot [fill=colorbl] fill between[of=upper and lower];
    \end{axis}
    \end{tikzpicture}
    \caption{AG News}
    \label{fig:corr_all_agn}
\end{subfigure}
\caption{Compare the performance of $\sul$-ICL ($y$-axis) against $\taskg$-ICL ($x$-axis). The \textcolor{colorl}{dots} represent the mean values, and the \textcolor{colorbl}{error bars} represent standard deviations. The dashed horizontal line represents the performance of random guessing (i.e., Chance). The Pearson correlation coefficients $\pearson$ and Spearman correlation coefficients $\spearman$ are also reported. Significant correlations ($p$-value $<0.01$) are observed.}
\label{fig:corr_all}
\vspace{-15pt}
\end{figure*}

\subsubsection{Experimental Setup}
\paragraph{Tasks.} In addition to CR, SST-2, and AG News, we also consider DBpedia \citep{DBLP:conf/nips/ZhangZL15}, which has many more classes (\cref{tab:nlp_datasets}).

\paragraph{Experimental Setup.} 
We consider the task $\taskg$ induced by the string-to-string function $\g$, as introduced in \cref{sec:h1}.
Given string--label pairs $\{(\bx^{(\ell)},\y^{(\ell)})\}_{\ell=1}^L$, we construct a demonstration for task $\taskg$ as follows: $\classtoresponse(\y^{(1)}) \g(\classtoresponse(\y^{(1)}))\delim \cdots \delim\classtoresponse(\y^{(L)}) \g(\classtoresponse(\y^{(L)}))$, i.e., $\bp_\ell=\classtoresponse(\y^{(\ell)})$ and $\br_\ell=\g(\classtoresponse(\y^{(\ell)}))$.
In this formulation, both prompts and responses consist of a single token, i.e., $\bp, \br\in\alphabet$. We denote the set of distinct responses as $\responseset\subset\alphabet$.

The single-token construction allows us to model the task using linear regression.
The embedding 
layer ($0^{\text{th}}$ layer) of an LLM is a matrix $\embedding\in\R^{|\alphabet|\times D}$,  where $D$ is the dimensionality of embeddings. We retrieve the row vector $\bfp_{\ell} \defeq \embedding_{\bp_{\ell},:} \in\R^{1\times D}$ for each prompt $\bp_{\ell}$, where $\embedding_{\bp_{\ell},:}$ denotes the row vector in $\embedding$ that corresponds to the token $\bp_{\ell}$.
We vertically stack the row vectors to create the embedding matrix $\bfP\in\R^{L\times D}$ of the prompts in the demonstration.\footnote{The rows of $\bfP$ are linearly independent.}
Also, let $\bfR\in\R^{L\times D}$ be a similarly constructed embedding matrix of the responses, i.e., each response $\br_{\ell}$ is embedded as a row vector $\bfr_{\ell} \in\R^{1\times D}$.
Additionally, for a test prompt--response pair $\langle \bp,\br\rangle=\langle\classtoresponse(\y),\g(\classtoresponse(\y))\rangle$, we embed the prompt 
$\bp$ as a row vector $\bfp\in\R^{1\times D}$. 
Thus, learning $\taskg$ is reduced to performing a multiple linear regression, i.e., learning a parameter matrix $\bW^\star \in \R^{D\times D}$  that minimizes the following (non-strictly) convex objective
\begin{equation}
    \bW^\star \in \argmax_{\bW} ||\bfR-\bfP\bW||^2.
\end{equation} 
Moreover, the data are constructed such that the test prompt--response pair $\langle \bp,\br\rangle$ has appeared at least once in $\{\langle \bp_\ell,\br_\ell\rangle\}_{\ell=1}^L$, so the task does not require generalization at all.

\paragraph{Settings.} 
We compare the following two settings:
\begin{itemize}[leftmargin=*]
    \item \textbf{$\taskg$-ICL}: We construct a demonstration $\demonstration=\bp_1\br_1\delim\cdots\delim\bp_L\br_L$ and perform classification as follows:
    \begin{equation}
        \br^\star =  \argmax_{\widetilde{\br}\in \responseset} \lm(\widetilde{\br} \mid \demonstration\delim\bp).
    \end{equation}
    The classification is correct if $\br^\star=\br$.
    \item \textbf{$\taskg$-Linear}: %
    We train a linear regression with gradient descent. 
    At inference time, we use a minimzer $\bW^\star$ to compute the predicted embedding vector $\bfr^\star$ for $\br$ as follows
    \begin{equation}
        \bfr^\star = \bfp\bW^\star\in\R^{1\times D}.
    \end{equation}
    In order to evaluate its performance against ICL, we utilize the transposed embedding layer 
    $\embedding^\top\in\R^{D\times |\alphabet|}$ to project $\bfr^\star$ into $\R^{1\times |\alphabet|}$ and perform classification as follows:
    \begin{equation}
       \br^\star =  \argmax_{\widetilde{\br}\in \responseset} \left(\bfr^\star \embedding^\top\right)_{\widetilde{\br}}.
    \end{equation}
    where $(\bfr^\star\embedding^\top)_{\widetilde{\br}}$ 
    denotes the entry of the vector $(\bfr^\star \embedding^\top)$ that corresponds to the token $\widetilde{\br}$. 
    We judge classification correct if $\br^\star=\br$.
    Note that the embedding layer $\embedding$ is taken directly from the LLM and not trained with the linear model.\looseness=-1
\end{itemize}

\paragraph{Implementation Details.} For all the following experiments in \cref{sec:h2_2} and \cref{sec:h3}, we experiment on the largest model (LLaMA2-70B) and set the demonstration length $L=32$. 
We randomly sample $500$ functions $\g$ for $\taskg$. Each mapping $\g$ is constructed by randomly selecting a token from $\alphabet$ as the image $\g(\classtoresponse(\y))$ of a $\classtoresponse(\y)$. 
We train $\taskg$-Linear
with also $32$ examples for $80$ epochs which corresponds to the $80$ layers of LLaMA2-70B. We choose a learning rate of $1000$, which we find yields the best performance. The correlation between the F1-Macro scores of $\taskg$-ICL and $\taskg$-Linear is computed.

\subsubsection{Results}
As shown in \cref{tab:linear_corr}, $\taskg$-Linear can learn most of the functions perfectly. On the other hand, $\taskg$-ICL has a much lower average and higher variance. In the same table, we also list the correlation between the performance of $\taskg$-Linear and $\taskg$-ICL. 
In contrast to \cref{pred:learn_1}, there does not exist any significant correlation, which again gives us evidence against \cref{hyp:learn}. 
It is worth mentioning that the high variance of $\taskg$-ICL's performance also goes against the claim of \citet{olsson2022incontext} that there exists a special kind of heads called induction heads 
that copy any abstract pattern.\looseness=-1

\section{Testing \cref{hyp:compose}}
\label{sec:h3}
\setcounter{prediction}{2}
Because neither \cref{hyp:recognize} nor \cref{hyp:learn} matches our experimental findings, we now turn to \cref{hyp:compose} to explain the empirical facts.

\subsection{Experiment 1}
We first examine the following prediction.
\begin{prediction}
    \label{pred:compose}
   Consider $\sul = \taskg \circ \task$.
   Then, an LLM's ability to learn $\sul$ in context correlates with its ability to learn $\taskg$ in context.
\end{prediction}
We verify it by comparing the performance of $\sul$-ICL and $\taskg$-ICL.
\subsubsection{Experimental Setup}
\label{sec:h3_1_exp}
The experimental setup of $\sul$-ICL and $\taskg$-ICL follow that of \cref{sec:h1_exp} and \cref{sec:h2_2}, %
respectively. 
The correlation between the F1-Macro scores of $\sul$-ICL and $\taskg$-ICL is computed.\looseness=-1
\subsubsection{Results}

For each dataset in CR, SST-2, and AG News, we bucket the $500$ data points for $500$ functions
into $8$ bins to better visualize the results.
The first group contains all functions $\g$ with an F1-Macro of $100\%$ on $\taskg$-ICL. 
The rest of the data points are put into $8-1=7$ bins evenly distributed and by increasing $\taskg$ performance.
The mean and the standard deviation of $\sul$-ICL's performance of each group are reported in \cref{fig:corr_all}. 
We compute the Pearson correlation coefficient (denoted as $\pearson$) and Spearman's rank correlation coefficient (denoted as $\spearman$) between the $\taskg$-ICL scores and the $\sul$-ICL scores of all the data points. 
There exists a modest positive correlation with $0.34\leq \pearson \leq 0.42$ and $0.35\leq \spearman\leq 0.51$, but it is statistically significant with $p$-value smaller than $0.01$ under Student's $t$-test.
We take it as evidence supporting that $\sul$ is learned via composing $\taskg$ and $\task$.\looseness=-1

\subsection{Experiment 2}
\label{sec:h3_2}
Next, rather than creating string-to-string functions $\g$ randomly, we construct hand-crafted functions that are intuitive and likely to have been in the pre-training data.  
We call such functions \defn{natural}. 
One example of such a natural function is $\syn$, where the prompts are mapped to synonyms. 
We compare ICL performance on these hand-crafted functions against that on random functions. 
If \cref{hyp:compose} were true, $\taskg$-ICL
should have a significantly higher performance on such natural functions.

\subsubsection{Experimental Setup}
\label{sec:h3_2_exp}
We consider three types of natural functions: (1) $\synonym$: Each prompt $\classtoresponse(\y)$ is mapped to one of its synonyms, (2) $\antonym$: Each prompt $\classtoresponse(\y)$ is mapped to one of its antonyms, and (3) $\keyword$: Each prompt $\classtoresponse(\y)$ is mapped to a keyword in its genre. Synonyms and antonyms are selected using PyMultiDictionary library.\footnote{\url{https://github.com/ppizarror/PyMultiDictionary}. The library aggregates information from \url{educalingo.com}, \url{synonym.com}, and WordNet \citep{miller-1994-wordnet}} Keywords are obtained for each genre by querying GPT-4 \citep{DBLP:journals/corr/abs-2303-08774}.
In contrast to $\random$, we cannot create a large number of natural functions as easily. 
Therefore, we manually choose a candidate set of $10$ possible synonyms (resp. antonyms and keywords) for each prompt. 
Thus, we can create a natural function by sampling a synonym (resp. antonym and keyword), from the candidate sets for every input.
We create 500 such functions.
\looseness=-1

\subsubsection{Results}
We plot the mean F1 Macro scores as well as the standard deviations of $\taskg$-ICL across different functions in \cref{tab:compose_mapping}. 
We observe that the LLM can learn $\synonym$ and $\keyword$ in context almost perfectly. 
The function $\antonym$ appears to be more difficult to learn, but still clearly much easier than $\random$, which is most evident in DBPedia where the LLM has an F1-Macro score of $84.5\%$ on $\antonym$ and only $59.8\%$ on $\random$. We also perform a one-sided Welch's $t$-test between each of $\antonym$, $\synonym$, $\keyword$ and $\random$. 
The $t$-values are all greater than $2$ and the $p$-values are all smaller than $0.01$. 
In other words, the natural functions are indeed significantly easier to learn in context, which is in line with our prediction. 

\begin{table}
    \small
    \centering
    \resizebox{1.\columnwidth}{!}{
    \smallskip\begin{tabular}{ccccc}
    \toprule
    \midrule
        \multirow{2}{*}{Dataset} & \multirow{2}{*}{Mapping} & \multirow{2}{*}{F1-Macro ($\%$)} & \multicolumn{2}{c}{$t$-test} \\
        \cmidrule{4-5}
        & & & $t$-value & $p$-value \\
        \midrule
          & $\random$ & $93.2 \pm 13.6$ & N.A. & N.A.  \\
         CR / & $\antonym$ & $97.0\pm 8.5$ & $4.35$ & < $0.01$ \\
         SST-2& $\synonym$ & $100.0\pm 0.1$ & $10.71$ & < $0.01$ \\
         & $\keyword$ & $100.0\pm 0.0$ & $10.75$ & < $0.01$ \\
         \midrule
         \multirow{4}{*}{AG News} & $\random$ & $93.8\pm 10.8$ & N.A. & N.A.  \\
         & $\antonym$ & $99.9\pm 0.3$ & $12.59$ & < $0.01$ \\
         & $\synonym$ & $100.0\pm 0.0$ & $12.72$ & < $0.01$ \\
         & $\keyword$ & $100.0\pm 0.0$ & $12.73$ & < $0.01$ \\
         \midrule
         \multirow{4}{*}{DBPedia} & $\random$ & $59.8\pm 19.7$ & N.A. & N.A.  \\
         & $\antonym$ & $84.5\pm 20.5$ & $19.42$ & < $0.01$ \\
         & $\synonym$ & $95.8\pm 1.7$ & $40.67$ & < $0.01$ \\
         & $\keyword$ & $93.3\pm 4.9$ & $36.90$ & < $0.01$ \\
         \midrule
         \bottomrule
    \end{tabular}
    }
    \caption{The performance of $\taskg$-ICL with different types of functions $\g$. One-sided $t$-tests are performed between the natural functions ($\antonym,\synonym,\keyword$ and the $\random$ functions. The LLM learns the natural functions significantly better. }
    \label{tab:compose_mapping}
    \vspace{-15pt}
\end{table}

\begin{figure*}
	\centering
	\includegraphics[width=.9\textwidth]{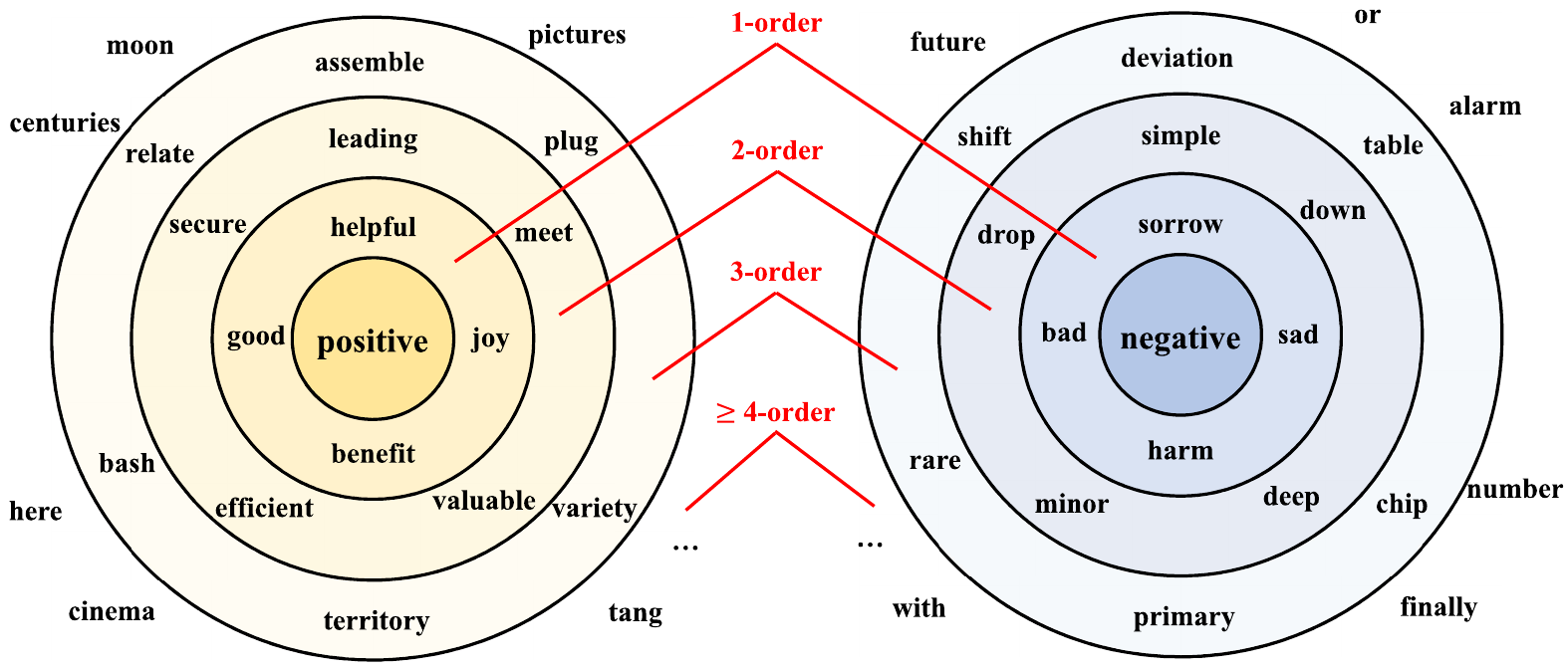}
    \vspace{-.2cm}
	\caption{Synonyms of \textexample{positive} and \textexample{negative}. 
 The words between concentric circles represent elements of the candidate sets of high-order synonyms. 
 As the order gets higher, the synonyms become less related to the seed word.}
	\label{fig:syn_posneg}
\end{figure*}

\begin{figure*}[!t]
\centering
\begin{subfigure}{.325\textwidth}
    \centering
    \begin{tikzpicture}
    \scalefont{0.8}
    \begin{axis}[
    sharp plot,
    xlabel=Order of $\synonym$,
    ylabel= $\taskg$-ICL (F1-Macro),
    width=5.5cm, height=4.5cm,
    ymin=.77, ymax=1.08,
    xlabel near ticks,
    ylabel near ticks,
    xmajorgrids=true,
    ymajorgrids=true,
    grid style=dashed,
    legend style={at={(0.9,1.1)}, anchor=south},
    legend columns=1, %
    legend pos=north east,
    legend style={nodes={scale=0.8, transform shape}},]
    \node[below right, align=center, text=black]
        at (rel axis cs:0., .35) {\,\,$\pearson$: -0.27 ($p$-value < $0.01$) \\ 
        \,\,$\spearman$: -0.31 ($p$-value < $0.01$)};
    \addplot+[line width=0.36mm, mark=o, mark options={scale=.9}, color=colorl] plot coordinates {
        (1, 0.9999609375)
        (2, 0.9949609375)
        (3, 0.9993359375)
        (4, 0.9823046875)
        (5, 0.974921875)
        (6, 0.92765625)};
    \addplot [name path=upper,draw=none] coordinates {
        (0.5, 0.9999609375 + 0.0003886669676197733)
        (1, 0.9999609375 + 0.0003886669676197733)
        (2, 0.9949609375 + 0.025557116400925763)
        (3, 0.9993359375 + 0.0031276843939196846)
        (4, 0.9823046875 + 0.08303047484760019)
        (5, 0.974921875 + 0.0689188710805366)
        (6, 0.92765625 + 0.1431271322166503)
        (6.5, 0.92765625 + 0.1431271322166503)};
    \addplot [name path=lower,draw=none] coordinates {
        (0.5, 0.9999609375 - 0.0003886669676197733)
        (1, 0.9999609375 - 0.0003886669676197733)
        (2, 0.9949609375 - 0.025557116400925763)
        (3, 0.9993359375 - 0.0031276843939196846)
        (4, 0.9823046875 - 0.08303047484760019)
        (5, 0.974921875 - 0.0689188710805366)
        (6, 0.92765625 - 0.1431271322166503)
        (6.5, 0.92765625 - 0.1431271322166503)};
    \addplot [fill=colorbl] fill between[of=upper and lower];
    \end{axis}
    \end{tikzpicture}
    \caption{CR / SST-2}
    \label{fig:syn_hops_cr}
\end{subfigure}
\begin{subfigure}{.325\textwidth}
    \centering
    \begin{tikzpicture}
    \scalefont{0.8}
    \begin{axis}[
    sharp plot,
    xlabel=Order of $\synonym$,
    width=5.5cm, height=4.5cm,
    ymin=.9, ymax=1.05,
    xlabel near ticks,
    ylabel near ticks,
    xmajorgrids=true,
    ymajorgrids=true,
    grid style=dashed,
    legend style={at={(0.9,1.1)}, anchor=south},
    legend columns=1, %
    legend pos=north east,
    legend style={nodes={scale=0.8, transform shape}},]
    \node[below right, align=center, text=black]
        at (rel axis cs:0., .35) {\,\,$\pearson$: -0.23 ($p$-value < $0.01$) \\ 
        \,\,$\spearman$: -0.22 ($p$-value < $0.01$)};
    \addplot+[line width=0.36mm, mark=o, mark options={scale=.9}, color=colorl] plot coordinates {
        (1, 0.999921875)
        (2, 0.9996484375)
        (3, 0.99828125)
        (4, 0.9997265625)
        (5, 0.9957421875)
        (6, 0.973203125)};
    \addplot [name path=upper,draw=none] coordinates {
        (0.5, 0.999921875 + 0.0005468749999999999)
        (1, 0.999921875 + 0.0005468749999999999)
        (2, 0.9996484375 + 0.0020739943675716064)
        (3, 0.99828125 + 0.008049529295244537)
        (4, 0.9997265625 + 0.001139527511391783)
        (5, 0.9957421875 + 0.016865729151869444)
        (6, 0.973203125 + 0.06739163212208267)
        (6.5, 0.973203125 + 0.06739163212208267)};
    \addplot [name path=lower,draw=none] coordinates {
        (0.5, 0.999921875 - 0.0005468749999999999)
        (1, 0.999921875 - 0.0005468749999999999)
        (2, 0.9996484375 - 0.0020739943675716064)
        (3, 0.99828125 - 0.008049529295244537)
        (4, 0.9997265625 - 0.001139527511391783)
        (5, 0.9957421875 - 0.016865729151869444)
        (6, 0.973203125 - 0.06739163212208267)
        (6.5, 0.973203125 - 0.06739163212208267)};
    \addplot [fill=colorbl] fill between[of=upper and lower];
    \end{axis}
    \end{tikzpicture}
    \caption{AG News}
    \label{fig:syn_hops_agn}
\end{subfigure}
\begin{subfigure}{.325\textwidth}
    \centering
    \begin{tikzpicture}
    \scalefont{0.8}
    \begin{axis}[
    sharp plot,
    xlabel=Order of $\synonym$,
    width=5.5cm, height=4.5cm,
    ymin=.63, ymax=1.,
    xlabel near ticks,
    ylabel near ticks,
    xmajorgrids=true,
    ymajorgrids=true,
    grid style=dashed,
    legend style={at={(0.9,1.1)}, anchor=south},
    legend columns=1, %
    legend pos=north east,
    legend style={nodes={scale=0.8, transform shape}},]
    \node[below right, align=center, text=black]
        at (rel axis cs:0., .35) {\,\,$\pearson$: -0.56 ($p$-value < $0.01$) \\ 
        \,\,$\spearman$: -0.71 ($p$-value < $0.01$)};
    \addplot+[line width=0.36mm, mark=o, mark options={scale=.9}, color=colorl] plot coordinates {
        (1, 0.9590234375)
        (2, 0.90265625)
        (3, 0.9014453125)
        (4, 0.8823828125)
        (5, 0.8433203125)
        (6, 0.8060546875)};
    \addplot [name path=upper,draw=none] coordinates {
        (0.5, 0.9590234375-0.016923532279764137)
        (1, 0.9590234375-0.016923532279764137)
        (2, 0.8853653677386751)
        (3, 0.8833087405609809)
        (4, 0.8428740529948381)
        (5, 0.7859824833821373)
        (6, 0.6578513617308291)
        (6.5, 0.6578513617308291)};
    \addplot [name path=lower,draw=none] coordinates {
        (0.5, 0.9590234375+0.016923532279764137)
        (1, 0.9590234375+0.016923532279764137)
        (2, 0.919947132261325)
        (3, 0.9195818844390192)
        (4, 0.9218915720051618)
        (5, 0.9006581416178626)
        (6, 0.9542580132691708)
        (6.5, 0.9542580132691708)};
    \addplot [fill=colorbl] fill between[of=upper and lower];
    \end{axis}
    \end{tikzpicture}
    \caption{DBpedia}
    \label{fig:syn_hops_dbp}
\end{subfigure}
\caption{The performance of LLaMA2-70B learning $\synonym$ of different orders. The performance in general decreases as the order increases. }
\label{fig:syn_hops}
\vspace{-15pt}
\end{figure*}
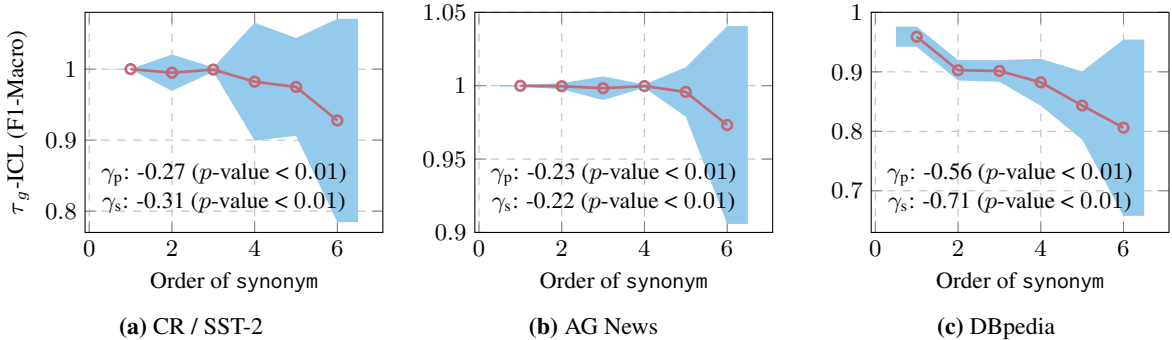

\subsection{Experiment 3}
For \cref{hyp:compose} to be a good hypothesis, however, we need to additionally show why seemingly arbitrary tasks $\taskg$ would likely be elements of $\obtaskalphabet^*$.
We offer a tentative explanation.
While we cannot show how to construct an arbitrary string-to-string function $\g$ out of natural functions, we can exhibit compositions of natural functions that appear arbitrary. 
We do so by composing natural functions, e.g., $\syn$, repeatedly as follows
\begin{equation}
        \syn[m](\cdot) = \underbrace{\syn(\syn(\cdots\syn(\cdot)))}_{\times m}.
\end{equation}
We call $\syn[m]$ the $m^{\text{th}}$ power of $\syn$.

\subsubsection{Experiment Setup}
The powers of the $\synonym$ functions are created the same way as in \cref{sec:h3_2_exp}. 
Because the candidate sets of synonyms are small (of size 10), to demonstrate how it is possible to create seemingly arbitrary functions, we adversarially sample functions such that for every $\bp$, $\syn[m](\bp) \in S^{m}(\bp) \setminus S^{m-1}(\bp)$, where $S^{m}(\bp)$ denotes the set of all the $m^{\text{th}}$ order synonyms of $\bp$, i.e., we only keep those that are \emph{not} synonyms of lower orders.
As an example of this strategy, we may have the following mapping: \textexample{company} is synonymous with \textexample{firm}, which is synonymous with \textexample{association} (as a noun) and \textexample{adamant} (as an adjective). 
Thus, we would map \textexample{adamant} to \textexample{company}, a seemingly unrelated word.\looseness=-1

\subsubsection{Results}
We visualize a small subset of the synonyms of the two $\classtoresponse(\y)$ of CR/SST in \cref{fig:syn_posneg}. The concentric circles are the candidate sets, with radii representing their $\synonym$ orders. %
As can be seen, taking a natural function to higher power results in functions that appear as if they were arbitrary string-to-string mappings, e.g., 
\begin{align*}
            \textexample{positive}&\rightarrow \textexample{tang} \\
            \textexample{negative}&\rightarrow \textexample{or}.
\end{align*}
Yet, as our results demonstrate, ICL is still able to learn these functions well---above $90\%$ F1-score---as shown in \cref{fig:syn_hops_agn}. 
The same also holds true for the AG News and DBPedia text classification datasets.
With a large task alphabet $\taskalphabet$, we expect an LLM to be able to learn many seemingly arbitrary string-to-string functions. 
More evidence demonstrating that these functions $\g$ are learned via composition comes from the noticeable performance decrease as the task composition becomes more complex. As shown in \cref{fig:syn_hops}, there exists a negative correlation ($p$-value $<0.01$) between the performance of $\taskg$-ICL and the order of $\synonym$.\looseness=-1

\section{Conclusion}
In this paper, we examine the ability of pre-trained LLMs to learn various tasks in context. 
We find that LLMs \emph{can} learn text classification tasks with corrupted responses ($\sul$-ICL), 
but when the prompts are corrupted in the same manner ($\sui$-ICL), ICL performs significantly worse than a logistic regression model. 
A closer look suggests the tasks in the $\sul$-ICL may have been learned via composing primitive tasks learned during pre-training. Overall, our paper provides insights into the nature, abilities, and limitations of ICL.

\section*{Acknowledgements}
This publication was made possible by an ETH AI Center doctoral fellowship to Jiaoda Li. Yifan Hou is supported by the Swiss Data Science Center PhD Grant (P22-05).\looseness=-1

\section*{Limitations}
One limitation is that we only experiment on one family of LLMs---LLaMA2. 
Our results might not hold on larger models, e.g., GPT-4.
Another limitation is that our formalization is not complete and is, sadly, a bit more vibes-based than we would have liked. 
We hope to achieve a more concrete formalization in future work.\looseness=-1

\section*{Ethical Considerations}
The datasets and pre-trained LLMs that we use are all publicly available. Our paper focuses on model interpretation, and we thereby do not foresee any ethical issues originating from this work.

\bibliography{custom}

\clearpage
\appendix
\onecolumn
\section{Related Work}
Many papers attempt to provide a theoretical grounding for ICL's emergence, hypothesizing that ICL emerges from pre-training on documents that are drawn from a mixture of latent concepts \citep{xie2022an, wang2023large, wies2023learnability} or a compositional attribute grammar \citep{hahn2023theory}. 
After pre-training, the LLM can recognize the concept or the generative process that the prompt is sampled from and use it for next-token prediction. 
Another stream of research endeavors to show how LLMs can learn new tasks in context. 
\citet{pmlr-v202-von-oswald23a, akyurek2023what, dai-etal-2023-gpt} show by construction that the transformer can implicitly implement various learning algorithms, such as gradient descent and least squares regression. 
Nevertheless, there is a gap between the settings assumed in the papers and the ones used in practice, e.g. the attention mechanism needs to be linear, or only linear regression problems are considered. 
Empirically, it has been shown that transformers can learn various function classes in context, e.g., linear regression \citep{raventos2023pretraining}, multi-layer perceptrons and decision trees \citep{garg2022what}, and regular languages \citep{akyürek2024incontext}, but the models are trained from scratch on synthetic data derived from functions in the same class, which do not conform with how LLMs are pre-trained in practice. It is argued that their performance might not be extrapolated to models with a larger size \citep{wei2022emergent} or longer training time \citep{singh2023transient}. Thus, we fix our attention on LLMs that are pre-trained on large natural text corpus and are practically used. Most related to our work are \citet{kossen2024incontext, shen2024pretrained}, which also find that ICL's behavior is different from models trained with conventional learning algorithms such as gradient descent. Our work is distinct in the investigation on the possibility of ICL learning via task composition. 

\section{Associativity of Task Composition}
\label{app:proof}
\begin{proposition}
    Task composition is associative, i.e., $(\task_1\circ\task_2)\circ\task_3=\task_1\circ(\task_2\circ\task_3)$.
\end{proposition}
\begin{proof}
The proof follows by simple manipulation:
\begin{equation}
\begin{aligned}
        \responselm_{\task_1 \circ (\task_2 \circ \task_3)} (\br \mid \bp) &= \sum_{\widetilde{\br}_1 \in \alphabet^*} \responselm_{\task_1}(\br \mid \widetilde{\br}_1)\left(\sum_{\widetilde{\br}_2 \in \alphabet^*} \responselm_{\task_2}(\widetilde{\br}_1 \mid \widetilde{\br}_2) \responselm_{\task_3} (\widetilde{\br}_2  \mid \bp)\right) \\
        &= \sum_{\widetilde{\br}_2 \in \alphabet^*}\left(\sum_{\widetilde{\br}_1 \in \alphabet^*} \responselm_{\task_1}(\br \mid \widetilde{\br}_1) \responselm_{\task_2}(\widetilde{\br}_1 \mid \widetilde{\br}_2) \right) \responselm_{\task_3} (\widetilde{\br}_2  \mid \bp) \\
        &= \responselm_{(\task_1 \circ \task_2) \circ \task_3} (\br \mid \bp).
\end{aligned}
\end{equation}
\begin{equation}
\begin{aligned}
    \promptlm_{\task_1 \circ (\task_2 \circ \task_3)} &= \promptlm_{\task_2 \circ \task_3}\\
    &=\promptlm_{\task_3} \\
    &= \promptlm_{(\task_1 \circ \task_2) \circ \task_3}
\end{aligned}
\end{equation}
\end{proof}

\section{Supplementary Results}

\subsection{Dataset Details} \label{appendix:template}
The prompt templating functions $\template(\bx)$ and the responses $(\classtoresponse(\y)$ for CR, SST-2, and AG News are in \cref{tab:nlp_datasets}. The altered responses are tokens from the \textit{Lorem Ipsum} generator.

\begin{table*}[!h]
	\small
	\centering
		\smallskip\begin{tabular}{llcc}
			\toprule
   \midrule
            \multirow{2}{*}{\textbf{Dataset}} & \textbf{Prompt} & \multicolumn{2}{c}{\textbf{Response}}  \\
            \cmidrule{3-4}
            & \textbf{Template} & \multicolumn{1}{c}{$\task$} & \multicolumn{1}{c}{$\sul$} \\
            \midrule
            \multirow{2}*{CR} &  {\color{ETHBlue}{Review}:} $\bx$ \textbackslash n  & \multirow{2}*{positive, negative} & \multirow{2}*{por, Ne} \\
            ~ & {\color{ETHBlue}{Sentiment}:} & ~ & ~ \\ 
            \midrule
            \multirow{2}*{SST-2} & {\color{ETHBlue}{Review}:} $\bx$ \textbackslash n & \multirow{2}*{positive, negative} & \multirow{2}*{por, Ne} \\
            ~ & {\color{ETHBlue}{Sentiment}:}  & & \\
            \midrule
            \multirow{2}*{AG News} & {\color{ETHBlue}{News}:} $\bx$ \textbackslash n & \multirow{2}*{word, sports, business, science} & \multirow{2}*{Mag, Am, Num, Lab} \\
            ~ & {\color{ETHBlue}{News type}:} &  &  \\ 
            \midrule
			\bottomrule
		\end{tabular}
    \caption{Dataset information, templates, and responses used for $\task$ and $\sul$.}
    \label{tab:nlp_datasets}
\end{table*}

\subsection{(First-Order) Synonym}
The $10$ candidate synonyms for each prompt are in \cref{tab:label_remapping_gpt}.

\begin{table*}[!h]
	\small
	\centering
		\smallskip\begin{tabular}{lll}
			\toprule
            \midrule
            \textbf{Dataset} & \textbf{Prompt} & \textbf{Synonyms} \\
            \midrule
            \multirow{2}*{CR / SST-2} & {positive} & good, bright, happy, cheer, benefit, fortune, helpful, joy, help, favorite \\
            ~ &  {negative} & bad, dark, dire, sorrow, harm, down, dim, bitter, sad, blue \\
            \midrule
            \multirow{6}*{AG News}  & world & earth, planet, universe, sphere, creation, domain, environment, habitat, society, system \\
            ~ & sports & game, play, exercise, competition, activity, challenge, contest, match, training, racing \\ 
            ~ & \multirow{1}*{business} & trade, commerce, industry, company, market, operation, firm, establishment, production,\\
            ~ & ~ & organization \\ 
            ~ & \multirow{1}*{science} & research, technology, knowledge, experiment, investigation, theory, discovery, analysis,\\
            ~ & ~ & discipline, learning \\ 
            \midrule
            \multirow{15}*{DBpedia} & company & firm, business, startup, establishment, operator, producer, chain, brand, office, Agency \\
            ~ & transport & vehicle, car, train, bus, plane, ship, tram, cab, Metro, carriage \\ 
            ~ & player & runner, footballer, basketball, race, box, golf, cycle, sky, board, sail \\ 
            ~ & politics & government, policy, state, nation, election, party, assembly, council, republic, leader \\ 
            ~ & artist & painter, writer, composer, singer, actor, designer, director, producer, poet, photograph \\ 
            ~ & animal & creature, pet, bird, fish, insect, species, habitat, conservation, wild, migration \\ 
            ~ & school & university, college, prep, primary, secondary, high, middle, grammar, technical, night \\ 
            ~ & plant & flower, tree, Fern, grass, leaf, bud, root, branch, seed, growth \\ 
            ~ & \multirow{1}*{village} & Township, settlement, community, district, parish, cluster, region, municipality,\\
            ~ & ~ & neighborhood, rural \\
            ~ & book & novel, volume, text, manual, guide, reference, edition, journal, cover, series \\ 
            ~ & nature & terrain, forest, mountain, river, sea, lake, ocean, beach, desert, garden \\ 
            ~ & album & record, release, compilation, single, track, score, collection, edition, session, live \\ 
            ~ & building & structure, house, stad, tower, hall, temple, palace, castle, fort, shed \\
            ~ & film & movie, picture, cinema, feature, animation, drama, comedy, western, mystery, horror \\
            \midrule
			\bottomrule
		\end{tabular}
    \caption{The first-order synonyms used for constructing $\syn$.}
    \label{tab:label_remapping_gpt}
\end{table*}

\subsection{Detailed Results} \label{appendix:icl_num}
We provide detailed results on CR, SST-2, and AG News in \cref{fig:icl_num_all}. We report LLaMA2 with three model sizes. The results on LLaMA2-7B and LLaMA2-13B are consistent with those on LLaMA2-70B that we report in the main text.
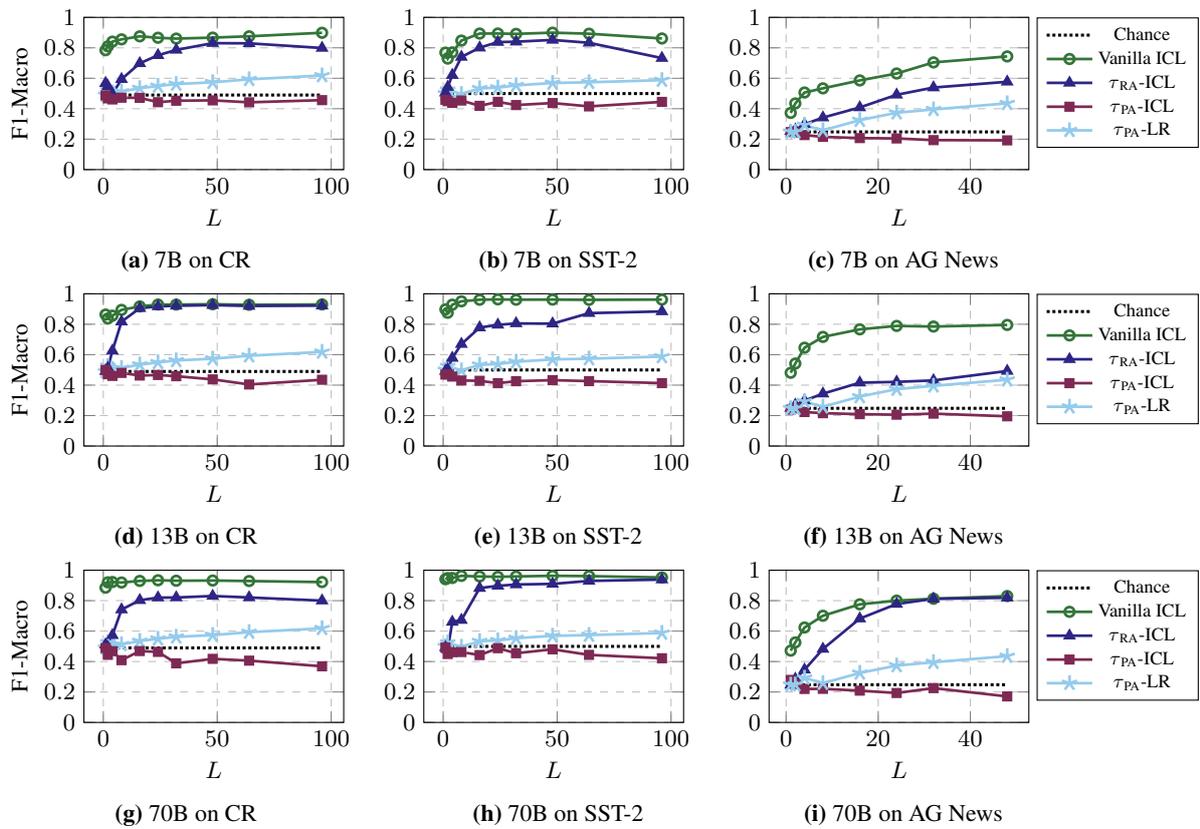
\begin{figure*}[!htbp]
\raggedright
\begin{subfigure}{.3\textwidth}
    \raggedright
    \begin{tikzpicture}
    \scalefont{0.8}
    \begin{axis}[
    sharp plot,
    xlabel=$L$,
    ylabel=F1-Macro,
    width=5cm, height=3.6cm,
    ymin=0., ymax=1.,
    xlabel near ticks,
    ylabel near ticks,
    xmajorgrids=true,
    ymajorgrids=true,
    grid style=dashed,
    legend style={at={(0.9,1.1)}, anchor=south},
    legend columns=3, %
    legend pos=outer north east, %
    legend style={nodes={scale=0.8, transform shape}},]
    \addplot+[line width=0.4mm, densely dotted, mark=., color=black,] plot coordinates {
        (1, 0.4898040493) 
        (96, 0.4898040493)
    };
    \addplot+[line width=0.36mm, mark=star, mark options={scale=1.5}, color=colorbl] plot coordinates {(1, 0.4898040493) (2, 0.5250976228) (4, 0.5069800903) (8, 0.5124265453) (16, 0.5345958961) (24, 0.5490063768) (32, 0.5621030572) (48, 0.5735139704) (64, 0.5930629944) (96, 0.6180121368)};
    \addplot+[line width=0.36mm, mark=o, mark options={scale=.9}, color=colorbo] plot coordinates {(1, 0.7843297589) (2, 0.8097375335) (4, 0.8405858704) (8, 0.8556076688) (16, 0.8754798759) (24, 0.8661446388) (32, 0.86060709) (48, 0.8670507042) (64, 0.8756946771) (96, 0.8988326259)};%
    \addplot+[line width=0.36mm, mark=triangle*, mark options={scale=.9}, color=colorsul] plot coordinates {(1, 0.5678736149) (2, 0.5480034343) (4, 0.4965714137) (8, 0.5939318224) (16, 0.6981754938) (24, 0.7510686685) (32, 0.7864122143) (48, 0.8306349074) (64, 0.8295236258) (96, 0.7987810098)};%
    \addplot+[line width=0.36mm, mark=square*, mark options={scale=.7}, color=colorsui,] plot coordinates {(1, 0.4881469187) (2, 0.4679160002) (4, 0.4609042241) (8, 0.4722829979) (16, 0.4713723352) (24, 0.4428742177) (32, 0.4524334092) (48, 0.4549103708) (64, 0.4420118903) (96, 0.4563576569)};%
    \end{axis}
    \end{tikzpicture}
    \caption{7B on CR}
    \label{fig:icl_num_cr_7b}
\end{subfigure}
\begin{subfigure}{.3\textwidth}
    \raggedright
    \begin{tikzpicture}
    \scalefont{0.8}
    \begin{axis}[
    sharp plot,
    xlabel=$L$,
    width=5cm, height=3.6cm,
    ymin=0., ymax=1.,
    xlabel near ticks,
    ylabel near ticks,
    xmajorgrids=true,
    ymajorgrids=true,
    grid style=dashed,
    legend style={at={(0.9,1.1)}, anchor=south},
    legend columns=3, %
    legend pos=north west, %
    legend style={nodes={scale=0.8, transform shape}},]
    \addplot+[line width=0.4mm, densely dotted, mark=., color=black,] plot coordinates { 
        (1, 0.4995513957)
        (96, 0.4995513957)
    };
    \addplot+[line width=0.36mm, mark=star, mark options={scale=1.5}, color=colorbl] plot coordinates {(1, 0.4995513957) (2, 0.523947331) (4, 0.475096686) (8, 0.4964842634) (16, 0.5333456) (24, 0.5384965662) (32, 0.553417537) (48, 0.5689058691) (64, 0.5738360081) (96, 0.5874644671)};
    \addplot+[line width=0.36mm, mark=o, mark options={scale=.9}, color=colorbo] plot coordinates {(1, 0.7685650685) (2, 0.7293553215) (4, 0.7700657177) (8, 0.8474842455) (16, 0.8937095144) (24, 0.8939937996) (32, 0.8915652842) (48, 0.8995990871) (64, 0.8926665881) (96, 0.8614254382)};%
    \addplot+[line width=0.36mm, mark=triangle*, mark options={scale=.9}, color=colorsul] plot coordinates {(1, 0.5055991069) (2, 0.542141854) (4, 0.6205467037) (8, 0.7398721421) (16, 0.8007977661) (24, 0.8377764294) (32, 0.8401193683) (48, 0.8517770111) (64, 0.8330570516) (96, 0.7324921612)};%
    \addplot+[line width=0.36mm, mark=square*, mark options={scale=.7}, color=colorsui,] plot coordinates {(1, 0.4588633087) (2, 0.4495273847) (4, 0.4386878582) (8, 0.4538757146) (16, 0.4189493202) (24, 0.4451158786) (32, 0.4240966021) (48, 0.4373312134) (64, 0.4152964515) (96, 0.4448766699)};%
    \end{axis}
    \end{tikzpicture}
    \caption{7B on SST-2}
    \label{fig:icl_num_sst2_7b}
\end{subfigure}
\hspace{-.45cm}
\begin{subfigure}{.3\textwidth}
    \raggedright
    \begin{tikzpicture}
    \scalefont{0.8}
    \begin{axis}[
    sharp plot,
    xlabel=$L$,
    width=5cm, height=3.6cm,
    ymin=-.0, ymax=1.0,
    xlabel near ticks,
    ylabel near ticks,
    xmajorgrids=true,
    ymajorgrids=true,
    grid style=dashed,
    legend style={at={(0.9,1.1)}, anchor=south},
    legend columns=1, %
    legend pos=outer north east, %
    legend style={nodes={scale=0.8, transform shape}},]
    \addplot+[line width=0.4mm, densely dotted, mark=., color=black,] plot coordinates { 
        (1, 0.2477786384)  
        (48, 0.2477786384) 
    };
    \addlegendentry{Chance}
    \addplot+[line width=0.36mm, mark=o, mark options={scale=.9}, color=colorbo] plot coordinates {(1, 0.3723821896) (2, 0.434631559) (4, 0.5064687182) (8, 0.5337698995) (16, 0.5858065109) (24, 0.6313098214) (32, 0.7042553425) (48, 0.743601293)};%
    \addlegendentry{Vanilla ICL} 
    \addplot+[line width=0.36mm, mark=triangle*, mark options={scale=.9}, color=colorsul] plot coordinates {(1, 0.2530772699) (2, 0.2615104709) (4, 0.2978122481) (8, 0.3417328041) (16, 0.4082206556) (24, 0.4914519061) (32, 0.539076567) (48, 0.5769243166)};%
    \addlegendentry{$\sul$-ICL} 
    \addplot+[line width=0.36mm, mark=square*, mark options={scale=.7}, color=colorsui,] plot coordinates {(1, 0.2466110411) (2, 0.2435373569) (4, 0.226847629) (8, 0.2150658126) (16, 0.207126579) (24, 0.2047007657) (32, 0.1937510542)(48, 0.1920662431)};%
    \addlegendentry{$\sui$-ICL} 
    \addplot+[line width=0.36mm, mark=star, mark options={scale=1.5}, color=colorbl] plot coordinates {(1, 0.2477786384) (2, 0.2477786384) (4, 0.2929082464) (8, 0.25769409) (16, 0.3249263953) (24, 0.3735114171) (32, 0.3954446548) (48, 0.4353368869)};
    \addlegendentry{$\sui$-LR} 
    \end{axis}
    \end{tikzpicture}
    \caption{7B on AG News}
    \label{fig:icl_num_agn_7b}
\end{subfigure}
\begin{subfigure}{.3\textwidth}
    \raggedright
    \begin{tikzpicture}
    \scalefont{0.8}
    \begin{axis}[
    sharp plot,
    xlabel=$L$,
    ylabel=F1-Macro,
    width=5cm, height=3.6cm,
    ymin=0., ymax=1.,
    xlabel near ticks,
    ylabel near ticks,
    xmajorgrids=true,
    ymajorgrids=true,
    grid style=dashed,
    legend style={at={(0.9,1.1)}, anchor=south},
    legend columns=3, %
    legend pos=outer north east, %
    legend style={nodes={scale=0.8, transform shape}},]
    \addplot+[line width=0.4mm, densely dotted, mark=., color=black,] plot coordinates {(1,0.4898040493) (96, 0.4898040493) };
    \addplot+[line width=0.36mm, mark=star, mark options={scale=1.5}, color=colorbl] plot coordinates {(1, 0.4898040493) (2, 0.5250976228) (4, 0.5069800903) (8, 0.5124265453) (16, 0.5345958961) (24, 0.5490063768) (32, 0.5621030572) (48, 0.5735139704) (64, 0.5930629944) (96, 0.6180121368)};
    \addplot+[line width=0.36mm, mark=o, mark options={scale=.9}, color=colorbo] plot coordinates {(1, 0.8636930009) (2, 0.8377046793) (4, 0.8563207072) (8, 0.8942157729) (16, 0.9169693006) (24, 0.9292443916) (32, 0.9291698849) (48, 0.9312596564) (64, 0.927845845) (96, 0.9301021654)};%
    \addplot+[line width=0.36mm, mark=triangle*, mark options={scale=.9}, color=colorsul] plot coordinates {(1, 0.4922902509) (2, 0.5132161738) (4, 0.625579586) (8, 0.8169477544) (16, 0.9051365828) (24, 0.9168752889) (32, 0.9213827385) (48, 0.9254780371) (64, 0.9200146364) (96, 0.922102915)};%
    \addplot+[line width=0.36mm, mark=square*, mark options={scale=.7}, color=colorsui,] plot coordinates {(1, 0.499040912) (2, 0.4715252821) (4, 0.4604736073) (8, 0.4802799496) (16, 0.4630596099) (24, 0.4665588158) (32, 0.4586228106) (48, 0.4377410307) (64, 0.4043907727) (96, 0.4352701266)};%
    \end{axis}
    \end{tikzpicture}
    \caption{13B on CR}
    \label{fig:icl_num_cr_13b}
\end{subfigure}
\begin{subfigure}{.3\textwidth}
    \raggedright
    \begin{tikzpicture}
    \scalefont{0.8}
    \begin{axis}[
    sharp plot,
    xlabel=$L$,
    width=5cm, height=3.6cm,
    ymin=0., ymax=1.,
    xlabel near ticks,
    ylabel near ticks,
    xmajorgrids=true,
    ymajorgrids=true,
    grid style=dashed,
    legend style={at={(0.9,1.1)}, anchor=south},
    legend columns=3, %
    legend pos=north west, %
    legend style={nodes={scale=0.8, transform shape}},]
    \addplot+[line width=0.4mm, densely dotted, mark=., color=black,] plot coordinates { 
        (1, 0.4995513957)
        (96, 0.4995513957)
    };
    \addplot+[line width=0.36mm, mark=star, mark options={scale=1.5}, color=colorbl] plot coordinates {(1, 0.4995513957) (2, 0.523947331) (4, 0.475096686) (8, 0.4964842634) (16, 0.5333456) (24, 0.5384965662) (32, 0.553417537) (48, 0.5689058691) (64, 0.5738360081) (96, 0.5874644671)};
    \addplot+[line width=0.36mm, mark=o, mark options={scale=.9}, color=colorbo] plot coordinates {(1, 0.8974744085) (2, 0.8753590706) (4, 0.9284038334) (8, 0.9505830194) (16, 0.960979019) (24, 0.9633188049) (32, 0.9611508825) (48, 0.9619284131) (64, 0.9601773371) (96, 0.9619156524)};%
    \addplot+[line width=0.36mm, mark=triangle*, mark options={scale=.9}, color=colorsul] plot coordinates {(1, 0.4931921316) (2, 0.5034674243) (4, 0.5782540167) (8, 0.6696641861) (16, 0.7780228474) (24, 0.7960918782) (32, 0.8052955284) (48, 0.8038346755) (64, 0.872896407) (96, 0.8840768254)};%
    \addplot+[line width=0.36mm, mark=square*, mark options={scale=.7}, color=colorsui,] plot coordinates {(1, 0.4696313822) (2, 0.4707722371) (4, 0.4563500037) (8, 0.4304679302) (16, 0.4282867683) (24, 0.4128664703) (32, 0.4258178282) (48, 0.4326422971) (64, 0.4263137209) (96, 0.4131573981)};%
    \end{axis}
    \end{tikzpicture}
    \caption{13B on SST-2}
    \label{fig:icl_num_sst2_13b}
\end{subfigure}
\hspace{-.45cm}
\begin{subfigure}{.3\textwidth}
    \raggedright
    \begin{tikzpicture}
    \scalefont{0.8}
    \begin{axis}[
    sharp plot,
    xlabel=$L$,
    width=5cm, height=3.6cm,
    ymin=-.0, ymax=1.0,
    xlabel near ticks,
    ylabel near ticks,
    xmajorgrids=true,
    ymajorgrids=true,
    grid style=dashed,
    legend style={at={(0.9,1.1)}, anchor=south},
    legend columns=1, %
    legend pos=outer north east, %
    legend style={nodes={scale=0.8, transform shape}},]
    \addplot+[line width=0.4mm, densely dotted, mark=., color=black,] plot coordinates { 
        (1, 0.2477786384)
        (48, 0.2477786384)
    };
    \addlegendentry{Chance}
    \addplot+[line width=0.36mm, mark=o, mark options={scale=.9}, color=colorbo] plot coordinates {(1, 0.4812723793) (2, 0.5418778218) (4, 0.6457503578) (8, 0.7167311548) (16, 0.7664484713) (24, 0.7889748339) (32, 0.7848531036) (48, 0.7954498769)};%
    \addlegendentry{Vanilla ICL} 
    \addplot+[line width=0.36mm, mark=triangle*, mark options={scale=.9}, color=colorsul] plot coordinates {(1, 0.2496276393) (2, 0.2700734176) (4, 0.2983973694) (8, 0.3439778692) (16, 0.4159683143) (24, 0.421292575) (32, 0.4309111775) (48, 0.4928473931)};%
    \addlegendentry{$\sul$-ICL} 
    \addplot+[line width=0.36mm, mark=square*, mark options={scale=.7}, color=colorsui,] plot coordinates {(1, 0.2386687711) (2, 0.247498983) (4, 0.2233013392) (8, 0.2156420864) (16, 0.2093138695) (24, 0.2057976562) (32, 0.2125816697) (48, 0.1952921818)};%
    \addlegendentry{$\sui$-ICL} 
    \addplot+[line width=0.36mm, mark=star, mark options={scale=1.5}, color=colorbl] plot coordinates {(1, 0.2477786384) (2, 0.2477786384) (4, 0.2929082464) (8, 0.25769409) (16, 0.3249263953) (24, 0.3735114171) (32, 0.3954446548) (48, 0.4353368869)};
    \addlegendentry{$\sui$-LR} 
    \end{axis}
    \end{tikzpicture}
    \caption{13B on AG News}
    \label{fig:icl_num_agn_13b}
\end{subfigure}
\begin{subfigure}{.3\textwidth}
    \raggedright
    \begin{tikzpicture}
    \scalefont{0.8}
    \begin{axis}[
    sharp plot,
    xlabel=$L$,
    ylabel=F1-Macro,
    width=5cm, height=3.6cm,
    ymin=0., ymax=1.,
    xlabel near ticks,
    ylabel near ticks,
    xmajorgrids=true,
    ymajorgrids=true,
    grid style=dashed,
    legend style={at={(0.9,1.1)}, anchor=south},
    legend columns=3, %
    legend pos=outer north east, %
    legend style={nodes={scale=0.8, transform shape}},]
    \addplot+[line width=0.4mm, densely dotted, mark=., color=black,] plot coordinates { 
        (1, 0.4898040493)
        (96, 0.4898040493)
    };
    \addplot+[line width=0.36mm, mark=star, mark options={scale=1.5}, color=colorbl] plot coordinates {(1, 0.4898040493) (2, 0.5250976228) (4, 0.5069800903) (8, 0.5124265453) (16, 0.5345958961) (24, 0.5490063768) (32, 0.5621030572) (48, 0.5735139704) (64, 0.5930629944) (96, 0.6180121368)};
    \addplot+[line width=0.36mm, mark=o, mark options={scale=.9}, color=colorbo] plot coordinates {(1, 0.8851730964) (2, 0.9205265309) (4, 0.9223953824) (8, 0.9192026189) (16, 0.9295924492) (24, 0.9337266643) (32, 0.9308682038) (48, 0.932161613) (64, 0.9288226723) (96, 0.9220212383)};%
    \addplot+[line width=0.36mm, mark=triangle*, mark options={scale=.9}, color=colorsul] plot coordinates {(1, 0.5109917704) (2, 0.4932524376) (4, 0.5727940949) (8, 0.741265894) (16, 0.8026284498) (24, 0.8198860679) (32, 0.8205913419) (48, 0.8305165683) (64, 0.8208618017) (96, 0.8000396788)};%
    \addplot+[line width=0.36mm, mark=square*, mark options={scale=.7}, color=colorsui,] plot coordinates {(1, 0.4890676878) (2, 0.4441093738) (4, 0.4679928704) (8, 0.4088857662) (16, 0.4688911161) (24, 0.4634365612) (32, 0.3882487295) (48, 0.4176498116) (64, 0.4058879064) (96, 0.3687743493)};%
    \end{axis}
    \end{tikzpicture}
    \caption{70B on CR}
    \label{fig:icl_num_cr_70b}
\end{subfigure}
\begin{subfigure}{.3\textwidth}
    \raggedright
    \begin{tikzpicture}
    \scalefont{0.8}
    \begin{axis}[
    sharp plot,
    xlabel=$L$,
    width=5cm, height=3.6cm,
    ymin=0., ymax=1.,
    xlabel near ticks,
    ylabel near ticks,
    xmajorgrids=true,
    ymajorgrids=true,
    grid style=dashed,
    legend style={at={(0.9,1.1)}, anchor=south},
    legend columns=3, %
    legend pos=north west, %
    legend style={nodes={scale=0.8, transform shape}},]
    \addplot+[line width=0.4mm, densely dotted, mark=., color=black,] plot coordinates { 
        (1, 0.4995513957)
        (96, 0.4995513957)
    };
    \addplot+[line width=0.36mm, mark=star, mark options={scale=1.5}, color=colorbl] plot coordinates {(1, 0.4995513957) (2, 0.523947331) (4, 0.475096686) (8, 0.4964842634) (16, 0.5333456) (24, 0.5384965662) (32, 0.553417537) (48, 0.5689058691) (64, 0.5738360081) (96, 0.5874644671)};
    \addplot+[line width=0.36mm, mark=o, mark options={scale=.9}, color=colorbo] plot coordinates {(1, 0.9408225998) (2, 0.9484771912) (4, 0.948958084) (8, 0.9626997295) (16, 0.9587636902) (24, 0.9568076034) (32, 0.9591637387) (48, 0.9631488021) (64, 0.9611319219) (96, 0.9524214155)};%
    \addplot+[line width=0.36mm, mark=triangle*, mark options={scale=.9}, color=colorsul] plot coordinates {(1, 0.4931789017) (2, 0.4803575639) (4, 0.6600471199) (8, 0.6724065982) (16, 0.8825932413) (24, 0.897344605) (32, 0.9061587534) (48, 0.9100169232) (64, 0.929713742) (96, 0.9382894455)};%
    \addplot+[line width=0.36mm, mark=square*, mark options={scale=.7}, color=colorsui,] plot coordinates {(1, 0.4922824534) (2, 0.4494344244) (4, 0.4614640552) (8, 0.4623221936) (16, 0.4423864069) (24, 0.4875760426) (32, 0.4546503432) (48, 0.4803050954) (64, 0.4438302947) (96, 0.4214667072)};%
    \end{axis}
    \end{tikzpicture}
    \caption{70B on SST-2}
    \label{fig:icl_num_sst2_70b}
\end{subfigure}
\hspace{-.45cm}
\begin{subfigure}{.3\textwidth}
    \raggedright
    \begin{tikzpicture}
    \scalefont{0.8}
    \begin{axis}[
    sharp plot,
    xlabel=$L$,
    width=5cm, height=3.6cm,
    ymin=-.0, ymax=1.0,
    xlabel near ticks,
    ylabel near ticks,
    xmajorgrids=true,
    ymajorgrids=true,
    grid style=dashed,
    legend style={at={(0.9,1.1)}, anchor=south},
    legend columns=1, %
    legend pos=outer north east, %
    legend style={nodes={scale=0.8, transform shape}},]
    \addplot+[line width=0.4mm, densely dotted, mark=., color=black,] plot coordinates { 
        (1, 0.2477786384)
        (48, 0.2477786384)
    };
    \addlegendentry{Chance}
    \addplot+[line width=0.36mm, mark=o, mark options={scale=.9}, color=colorbo] plot coordinates {(1, 0.4715167392) (2, 0.5255636356) (4, 0.623556005) (8, 0.7007062289) (16, 0.7758450964) (24, 0.8000677889) (32, 0.8125316722) (48, 0.830045702)};%
    \addlegendentry{Vanilla ICL} 
    \addplot+[line width=0.36mm, mark=triangle*, mark options={scale=.9}, color=colorsul] plot coordinates {(1, 0.2468478101) (2, 0.2856521754) (4, 0.3462552517) (8, 0.4825528712) (16, 0.6819940028) (24, 0.7780325244) (32, 0.8120065361) (48, 0.8190959566)};%
    \addlegendentry{$\sul$-ICL} 
    \addplot+[line width=0.36mm, mark=square*, mark options={scale=.7}, color=colorsui,] plot coordinates {(1, 0.2794006883) (2, 0.2521906403) (4, 0.2199047002) (8, 0.2206811407) (16, 0.2088351781) (24, 0.1933858385) (32, 0.2253979441) (48, 0.1708924802)};%
    \addlegendentry{$\sui$-ICL} 
    \addplot+[line width=0.36mm, mark=star, mark options={scale=1.5}, color=colorbl] plot coordinates {(1, 0.2477786384) (2, 0.2477786384) (4, 0.2929082464) (8, 0.25769409) (16, 0.3249263953) (24, 0.3735114171) (32, 0.3954446548) (48, 0.4353368869)};
    \addlegendentry{$\sui$-LR} 
    \end{axis}
    \end{tikzpicture}
    \caption{70B on AG News}
    \label{fig:icl_num_agn_70b}
\end{subfigure}
\caption{ICL performance with different demonstration lengths $L$.}
\label{fig:icl_num_all}
\end{figure*}

\section{Experimental Setup}
\label{appendix:license}
We implement our experiments on A100-80G. Each experiment takes around 1-2 GPU hours. For LLaMA2, We use the implementation and pre-trained weights provided by HuggingFace \citep{wolf-etal-2020-transformers}.
CR \citep{DBLP:conf/kdd/HuL04} and SST \citep{socher-etal-2013-recursive} are under the CC-BY license. AG News~\citep{DBLP:conf/nips/ZhangZL15} and DBpedia~\citep{DBLP:conf/nips/ZhangZL15} are under the BSD-3-Clause license. We use these datasets consistently with their intended use.

\end{document}